\documentclass[10pt,twocolumn,letterpaper]{article}
\usepackage[pagenumbers]{iccv} 

\PassOptionsToPackage{table}{xcolor} 
\usepackage{colortbl}
\usepackage{times}
\usepackage{epsfig}
\usepackage{graphicx}
\usepackage{amsmath}
\usepackage{amssymb}
\DeclareMathOperator*{\argmax}{arg\,max}
\DeclareMathOperator*{\argmin}{arg\,min}
\usepackage{url}
\usepackage{graphicx}
\usepackage{booktabs}
\usepackage{amsmath}
\usepackage{amssymb}
\usepackage{amsthm}
\usepackage{pifont}
\usepackage{lipsum}
\usepackage[export]{adjustbox}
\usepackage{capt-of}
\usepackage[linesnumbered,ruled,vlined]{algorithm2e}
\usepackage{kotex}
\usepackage{thm-restate}
\usepackage{multirow}
\usepackage{caption,setspace}
\usepackage[figoff]{figcaps}
\newcommand{\cmark}{\ding{51}}
\usepackage{multicol}
\newcommand{\xmark}{\ding{55}}
\newcommand{\mbx}{\mathbf{x}}
\newcommand{\mbz}{\mathbf{z}}
\newcommand{\ent}{\alpha\texttt{-Entmax}}
\newcommand{\mbp}{\mathbf{p}}
\newcommand{\mbxi}{\mathbf{\xi}}
\newcommand{\mbXi}{\mathbf{\Xi}}
\newcommand{\mt}{\mathcal{T}}

\newcommand{\mbc}{\mathbf{c}}
\newcommand{\me}{\boldsymbol{\epsilon}}

\usepackage{booktabs}  
\newtheorem{thm}{Theorem}

\newtheorem{definition}{Definition}
\newtheorem{corollary}{Corollary}[thm]
\newtheorem{lem}{Lemma}
\usepackage{fontawesome5}
\usepackage{resizegather}
\newcommand{\frow}{\textcolor{orange}\faFrown}
\newcommand{\smil}{\textcolor{cyan}\faGrin}

\newcommand{\Att}{\texttt{At}}

\newcommand{\pb}{\mathbf{p}}
\newcommand{\xb}{\mathbf{x}}

\newcommand{\zb}{\mathbf{z}}
\newcommand{\Kb}{\mathbf{K}}
\newcommand{\Qb}{\mathbf{Q}}
\newcommand{\Vb}{\mathbf{V}}
\newcommand{\Wb}{\mathbf{W}}
\newcommand{\Xb}{\mathbf{X}}
\newcommand{\Yb}{\mathbf{Y}}
\newcommand{\Rb}{\mathbf{R}}

\newcommand{\Tc}{\mathcal{T}}
\newcommand{\Tca}{\mathcal{T}_\alpha}
\newcommand{\Tcal}{\mathcal{T}_\alpha^\lambda}

\newcommand{\xib}{\boldsymbol{\xi}}
\newcommand{\Xib}{\boldsymbol{\Xi}}
\newcommand{\Rd}{\mathbb{R}}
\newcommand{\norm}[1]{\left\lVert#1\right\rVert}

\definecolor{iccvblue}{rgb}{0.21,0.49,0.74}
\usepackage[pagebackref,breaklinks,colorlinks,allcolors=iccvblue]{hyperref}
\def\paperID{10639} 
\def\confName{ICCV}
\def\confYear{2025}

\title{PLADIS: Pushing the Limits of Attention in Diffusion Models \\ at Inference Time by Leveraging Sparsity}

\author{Kwanyoung Kim$^{1}$, Byeongsu Sim$^{1}$\\
Samsung Research $^{1}$\\
{\tt\small \{k_0.kim, bs.sim\}@samsung.com}
}

\author{Kwanyoung Kim$^\dagger$, \; Byeongsu Sim  \\
Samsung Research \\ 
{\tt\small \{k$\_$0.kim, bs.sim\}@samsung.com }
}

\begin{document}

\twocolumn[{%
	\renewcommand\twocolumn[1][]{#1}%
	\maketitle
	\begin{center}
		\centering
		\captionsetup{type=figure}
		\includegraphics[width=0.85\linewidth]{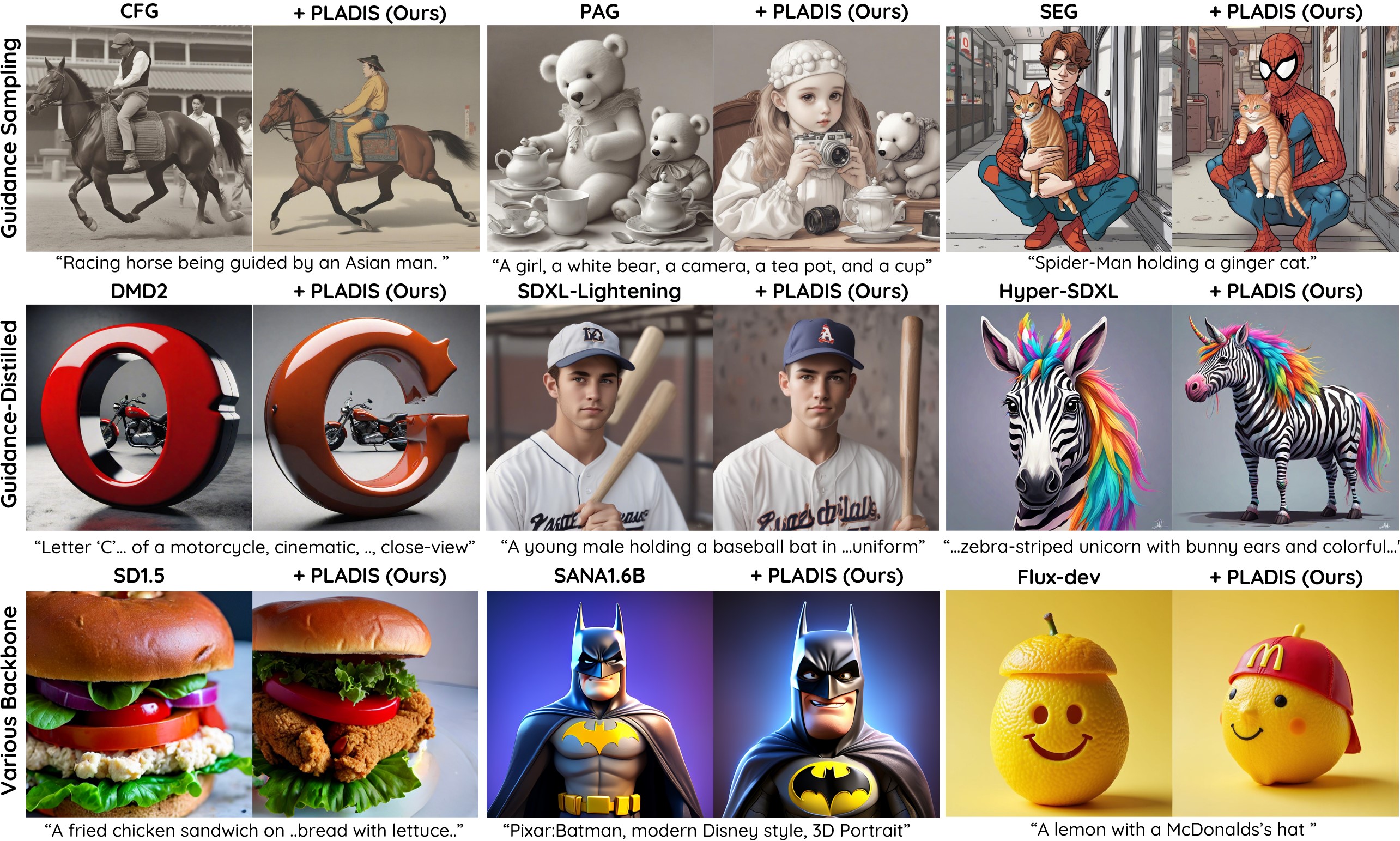} 
            \vspace{-0.5em}
		\captionof{figure}{Qualitative comparison (Top): guidance sampling methods (CFG\cite{CFG}, PAG\cite{PAG}, SEG\cite{SEG}) 
   (Mid): guidance-distilled models (DMD2\cite{dmd2}, SDXL-Lightning~\cite{light}, Hyper-SDXL\cite{hyper}) (Bottom): Other backbone such as Stable Diffusion 1.5~\cite{stablediffusion}, SANA~\cite{sana}, Flux~\cite{flux2024} with our method, PLADIS(Ours). PLADIS is compatible with all guidance techniques and also supports guidance-distilled models including various backbone. It provides the generation of plausible and improved text alignment without any training or extra inference.}\label{fig:main}
	\end{center}}
]
\def\thefootnote{$\dagger$}\footnotetext{First and corresponding author}

\begin{abstract}
Diffusion models have shown impressive results in generating high-quality conditional samples using guidance techniques such as Classifier-Free Guidance (CFG). However, existing methods often require additional training or neural function evaluations (NFEs), making them incompatible with guidance-distilled models. Also, they rely on heuristic approaches that need identifying target layers. In this work, we propose a novel and efficient method, termed PLADIS, which boosts pre-trained models (U-Net/Transformer) by leveraging sparse attention. Specifically, we extrapolate query-key correlations using softmax and its sparse counterpart in the cross-attention layer during inference, without requiring extra training or NFEs. By leveraging the noise robustness of sparse attention, our PLADIS unleashes the latent potential of text-to-image diffusion models, enabling them to excel in areas where they once struggled with newfound effectiveness. It integrates seamlessly with guidance techniques, including guidance-distilled models. Extensive experiments show notable improvements in text alignment and human preference, offering a highly efficient and universally applicable solution. See our project page: \url{https://github.com/cubeyoung/PLADIS}
\end{abstract}

\section{Introduction}
Diffusion models have demonstrated remarkable advancements in generating high-quality images and videos~\cite{stablediffusion,stablediffusion3,dreambooth,pixart2,sana,stablevideo,videocrafter2}. However, when using naïve sampling methods, the quality of the generated samples can be suboptimal. Classifier-Free Guidance (CFG)~\cite{CFG} is a prominent technique that increases the likelihood of a sample belonging to a specific class by calculating the difference between the score functions of conditional and unconditional models, and applying a weighted adjustment. While CFG is effective, it needs additional training and inference, and can degrade sample quality when the guidance scale is too high.

Inspired by CFG, various guidance sampling methods have been explored~\cite{SAG,AutoGuide,PAG,SEG,CFG++,TSG,SelfG}. Recent research has focused on creating "weak models" by intentionally weakening a model to guide the stronger, original model. Although these methods generally improve performance, they also come with clear limitations. For example, AutoGuidance (AG)~\cite{AutoGuide} relies on a poorly trained version of the unconditional model, which can be challenging and unstable to train. Alternative attention-based guided sampling methods, independent of the training process, have also been explored. For instance, Perturbed Attention Guidance (PAG)~\cite{PAG} disrupts self-attention maps by converting them into identity matrices, while Smooth Energy Guidance (SEG)~\cite{SEG} introduces blurring into attention weights. These methods are heuristic, as they are applied to specific layers, introducing additional hyperparameters that need to be determined through grid search. 

Furthermore, all existing guidance sampling methods require additional neural function evaluations (NFE{s}) and are not applicable to guidance-distilled models~\cite{lcm,tcd,light,turbo,hyper,dmd2,pcm} due to the need to calculate the difference between conditional and unconditional models or weak models. These limitations present a challenging and interesting problem: \textbf{\textit{Can we develop a universal boosting method that does not require additional training or NFE, can be combined with other guidance sampling methods, and can be applied to guidance-distilled models?}}

\begin{table}[t!]
\caption{Comparison of PLADIS with other sampling methods reveals key advantages of ours, with \smil ~  and \frow ~ denoting positive and  negative connotations for each category. }
\vspace{-0.5em}
\label{tab_sys}
\centering
\resizebox{1\linewidth}{!}{
\begin{small}
\begin{tabular}{lcccc}
\toprule
\multirow{2}{*}{\shortstack[c]{Method}} & \multirow{2}{*}{\shortstack[c]{Need extra \\  Training}} & \multirow{2}{*}{\shortstack[c]{Need heuristic \\ Search}} & \multirow{2}{*}{\shortstack[c]{Need extra  \\Inference}} & \multirow{2}{*}{\shortstack[c]{Supports guidance-\\ Distilled Model}} \\
\\
\cmidrule(lr){1-1} \cmidrule(lr){2-5}
CFG~\cite{CFG} & \frow & \smil & \frow & \frow \\
SAG~\cite{SAG} & \smil & \frow & \frow & \frow\\
AG~\cite{AutoGuide} & \smil & \frow & \frow & \frow \\
PAG~\cite{PAG} & \smil & \frow & \frow & \frow \\
SEG~\cite{SEG} & \smil & \frow & \frow & \frow  \\
\midrule
\textbf{PLADIS (Ours)} & \smil & \smil & \smil & \smil \\
\bottomrule
\end{tabular}   
\end{small}
}
\vspace{-1.5em}
\end{table}



In this work, we aim to tackle this challenging problem by adopting attention-based methods in a completely different
route.
One of the most important
contributions of this paper is the discovery of the importance of classical result from sparse attention via $\ent$~\cite{entmaxx} which includes softmax and sparsemax~\cite{sparsemax} as particular cases, and is sparse
for any $\alpha > 1$ and produce sparse alignment to assign nonzero probability.  Although widely investigated in natural language processing (NLP)~\cite{sparsemax,entmaxx,adapentmax,speeding}, sparse attention has not yet been extensively utilized within the realm of computer vision, particularly in diffusion models.
Specifically, our findings demonstrate that substituting cross-attentions with sparse counterparts during inference significantly improves overall generation performance.  Rather than weakening models via self-attention, which requires additional inference time, modifying the cross-attention mechanism circumvents the need for extra inference. This ensures compatibility with other guidance sampling methods and guidance-distilled models.

Interestingly, this result can be interpreted through the lens of modern Hopfield Networks~\cite{hopfield} and sparse Hopfield Networks (SHN)~\cite{sparsehop,stanhop}. In these works, the attention layer mirrors the update rule of Hopfield network to retrieve stored patterns. Moreover, there is a noise robustness advantage when we use sparse counterparts, which supports the rationale behind our approach in diffusion models.

Building on these findings and insights, we propose a novel and straightforward method, referred to as PLADIS, which assigns weights to the differences between sparse and dense attention to emphasize sparsity. As highlighted in Tab.~\ref{tab_sys}, our approach effectively addresses the aforementioned challenges, leading to improved performance and enhanced text-image alignment, as demonstrated by extensive experiments. Our key contributions are as follows:
\begin{itemize}
\item We propose a simple but effective method, named PLADIS, which substitutes cross-attention in diffusion models with adjusted attention mechanisms that extrapolating between sparse and dense cross-attentions.

\item We provide a thorough theoretical analysis based on our understanding of SHN, and propose the error bound and noise robustness of sparse attention for intermediate sparsity case. To the best of our knowledge, this is the first paper to apply and improve diffusion models from the perspective of SHN.

\item Our method can be combined with other guidance methods and even guidance-distilled models, does not require extra training or NFEs. We have demonstrated these advantages on various benchmark datasets, showing significant improvements in sample image quality, text-image alignment, and human preference evaluation.
\end{itemize}

\section{Preliminary}
\subsection{Diffusion Models}
Diffusion models (DM)~\cite{ho2020denoising, song2021scorebased} are a class of generative models designed to learn the reverse of a forward noise process by leveraging the score function of the data distribution. 
Specifically, given a data distribution \( \mbx_0 \sim q(\mbx_0) := q_{\text{data}}(\mathbf{x}) \), the forward process iteratively adds noise to the data according to a  Markov chain \( q(\mbx_{t}|\mbx_{t-1}) \sim \mathcal{N}(\sqrt{1-\beta_t}\mbx_{t-1}, \beta_t \mathbf{I}) \) for \(t = 1,\dots, T\) with pre-defined schedule $\{\beta_t\}_{t=1,\dots,T}$.
Consequently, the distribution of a latent variable is  $q(\mbx_t)=\mathcal{N}(\sqrt{\bar{\alpha}_t} \mbx_0, (1-\bar{\alpha}_t) \mathbf{I})$ and the distribution of last one approximates to an isotropic Gaussian distribution
\( q(\mbx_T) \approx \mathcal{N}(0, \mathbf{I})\), where \( \alpha_t = 1 - \beta_t \), $\bar{\alpha}_t$ = $\prod^{t}_{i} \alpha_i$.
The reverse process is modeled as \( p_\theta (\mbx_{t-1} | \mbx_t) = \mathcal{N}( \mu_\theta (\mbx_t, t), \Sigma_\theta (\mbx_t,t)) \). This model can be trained with variational bound on log likelihood~\cite{ho2020denoising} or trained with a score function in continuous time formulation~\cite{song2021scorebased}.
Both training objectives are reformulated with 
denoising score matching (DSM)~\cite{vincent2011connection}:
\begin{align}
\underset{\theta}{\min} \; \mathbb{E}_{\mbx_t = \sqrt{\bar{\alpha}_t}\mbx_0 + \sqrt{1-\bar{\alpha}_t} \me,\me \sim \mathcal{N}(0,I)} \; [\|\me_{\theta}(\mbx_t,t) -  \me \|^2_2]. \label{eq:dsm}
\end{align}

Sampling process is conducted as the learned reverse process starting from the isotropic Gaussian distribution.
For instance, given $\mbx_T \sim \mathcal{N}(0,\mathbf{I})$, DDIM~\cite{song2021ddim} samples $\mbx_0$ are computed as follow:
\begin{align}
&\mbx_{t-1} = \sqrt{\bar{\alpha}_{t-1}} \hat \mbx_0(t)  + \sqrt{1- \bar{\alpha}_{t-1}}\me_{\theta}(\mbx_t,t), 
\end{align}
where \( \hat \mbx_0(t) := \mathbb{E}[\mathbf{x}_0|\mathbf{x}_t] = (\mbx_t - \sqrt{1- \bar{\alpha_t}}\me_{\theta}(\mbx_t,t)) / \sqrt{\bar{\alpha_t}} \) is the denoised estimate by Tweedie's formula~\cite{efron2011tweedie,kim2021noise2score}. This process is repeated from $T$ to 1.

\subsection{Guidance Sampling in Diffusion Models}\label{sec:pre_guidance}
In order to generate samples following condition given by users, diffusion models are extended to conditional generative models \cite{CFG, rombach2022high} with additional inputs in the models:
\begin{align*}
\underset{\theta}{\min} \; \mathbb{E}_{\mbx_t, \me, \mbc} \; [\|\me_{\theta}(\mbx_t,t, \mbc) -  \me \|^2_2], \end{align*}
where $\mbx_t, \me$ are sampled same as Eq. \ref{eq:dsm} and $\mbc$ denotes a specific condition that $\mbx$ has, in most cases the embedding of a class or text.
However, since vanilla sampling often results in suboptimal performance for conditional generation, various guidance sampling methods have been extensively explored to enhance sample quality~\cite{CG,CFG,CFG++,SAG,AutoGuide,PAG,SEG,TSG}. For clarity, let us shorten the notation as $\me_{\theta}(\mbx_t,\mbc):=\me_{\theta}(\mbx_t,t,\mbc)$ and denote the unconditional model as $\me_\theta (\mbx_t, \mathbf{\varnothing} )$, where $\mathbf{\varnothing}$ represents the null condition.
Classifier-Free Guidance (CFG) adjusts the class-conditioned probability relative to the unconditional one, becoming $\hat p(\mbx_t|\mbc) = p(\mbx_t|\mbc) \left( \frac{p(\mbx_t|\mbc)}{p(\mbx_t|\mathbf{\varnothing})} \right)^w$, resulting in an adjusted sampling process:
\begin{align}
 &\mbx_{t-1} = \sqrt{\bar{\alpha}_{t-1}} \hat \mbx_0(t)  + \sqrt{1- \bar{\alpha}_{t-1}}\me'_{\theta}(\mbx_t,t),\\
 &\me'_{\theta}(\mbx_t,\mbc) = \me_{\theta}(\mbx_t,\mbc) + w(\me_{\theta}(\mbx_t,\mbc) - \me_{\theta}(\mbx_t,\mathbf{\varnothing})), \label{eq:cfg}
\end{align}
where $w$ is the guidance scale.
Recently, "weak model" guidance has been introduced, which weakens the conditional model and computes the difference with the normal conditional output as follow:
\begin{align}
 \me''_{\theta}(\mbx_t,\mbc) = \me_{\theta}(\mbx_t,\mbc) + s(\me_{\theta}(\mbx_t,\mbc)-\tilde{\me}_{\theta}(\mbx_t,\mbc)) \label{eq:weak}
\end{align}
where $s$ is the guidance weight, and $\tilde{\me}$ represents a model that is intentionally weakened or perturbed, achieved through various heuristic methods. For instance, AG~\cite{AutoGuide} uses a flawed model variant, PAG~\cite{PAG} replaces self-attention weights with an identity matrix, SEG~\cite{SEG} blurs attention weights, Time Step Gudiance (TSG)~\cite{TSG} perturbs timestep embeddings, and SelfGuidance~\cite{SelfG} alters noise levels. While effective, these approaches lack a clear theoretical foundation and have limitations: 1) they require specific layer identification, 2) increase computational cost with added NFEs, and 3) are incompatible with step-distilled models. Our method overcomes all of these limitations.

\begin{figure*}[ht!]
\centering
\includegraphics[width=0.75\linewidth]{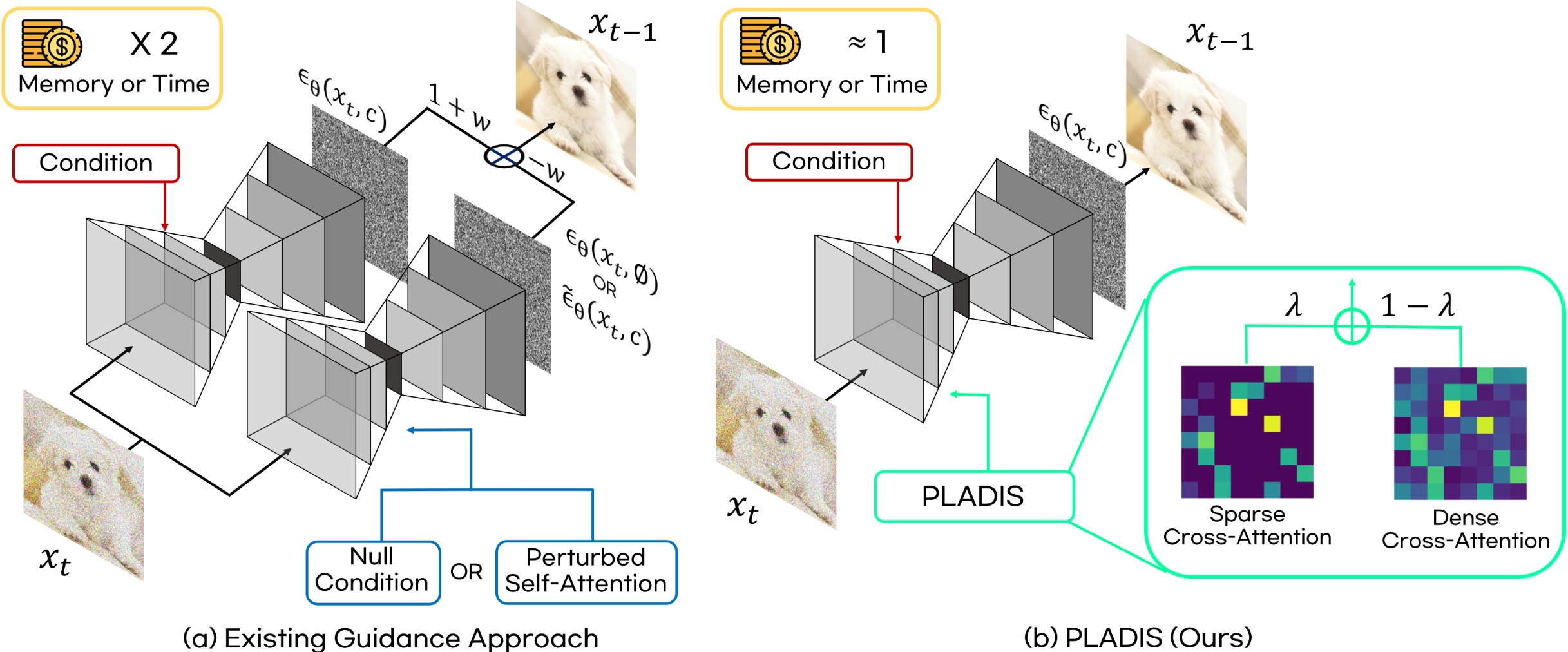}
\vspace{-0.5em}
\caption{Conceptual comparison between other guidance methods~\cite{CFG,PAG,SEG} and PLADIS: Existing guidance methods require extra inference steps due to undesired paths, such as null conditions or perturbing self-attention with an identity matrix or blurred attention weights. In contrast, PLADIS avoids additional inference paths by computing both sparse and dense attentions within all cross-attention modules using a scaling factor, $\lambda$. Moreover, PLADIS can be easily integrated with existing guidance approaches by simply replacing the cross-attention module.}
\label{fig:concept}
\vspace{-1em}
\end{figure*}

\subsection{Energy-Based Interpretations of Attention}~\label{sec:hop}
Attention mechanisms, following their distinct success, have recently been applied across various fields, including diffusion models~\cite{otseg,daam,p2p,park2023energy,attend,towards}. An energy-based model perspective has revealed their connection to Hopfield energy functions~\cite{hopfield, sparsehop, stanhop}. In Hopfield networks, the goal is to associate an input query $\mbx$ with the most relevant pattern $\mbxi$ by minimizing the energy function $E(\mbx)$ through retrieval dynamics $\mathcal{T}$. In modern Hopfield networks~\cite{hopfield}, energy functions and dynamics has been proposed, which is equivalent to attention mechanisms:
\begin{align}
E(\mbx)_{\texttt{Dense}} &:= -\texttt{lse}(\beta,\mbXi^{\top}\mbx) + \frac{1}{2} \langle \mbx,\mbx \rangle, \label{moder-energy2}\\
\mt_{\texttt{Dense}} (\mbx) &:= \mbXi\texttt{Softmax}(\beta\mbXi^{\top}\mbx) \label{attention}
\end{align}
where $\mbx \in \mathbb{R}^{d}$, $\mbXi =[\mbxi_1 \cdots, \mbxi_M] \in \mathbb{R}^{d \times M}$ , and $\texttt{lse}(\beta,\mbz):=\log\left(\sum^M_{i=1}\exp(\beta z_i)\right) / \beta$ denotes log-sum-exponential function for any given vector $\mbz \in \mathbb{R}^{M}$ and $\beta > 0 $. 
It mirrors the attention mechanism in transformers and providing a theoretical basis for its success.

Since sparse attention was introduced for its efficiency~\cite{sparsemax,entmaxx,adapentmax,speeding}, the Sparse Hopfield network~(SHN)~\cite{sparsehop,stanhop} was also proposed, extending the previous connection.
The energy function was modified to make sparse the computation of retrieval dynamics:
\begin{align}
E_{\alpha}(\mbx) := -\mathbf{\Psi}^{\star}_{\alpha}(\beta,\mbXi^{\top}\mbx) + \frac{1}{2}\langle \mbx,\mbx \rangle, \label{gshenergy} \\
 \mt_{\alpha} (\mbx) := \mbXi\alpha\texttt{-Entmax}(\beta\mbXi^{\top}\mbx), \label{sparse-ret}
\end{align}
and $\mathbf{\Psi}^{\star}_{\alpha}$ is the convex conjugate of Tsallis entropy~\cite{tsallis}, $\mathbf{\Psi}_\alpha, \alpha\texttt{-Entmax}(\mathbf{z})$, represents the probability mapping:
\begin{align}
&\Psi_{\alpha}(\mathbf{p}):=
\begin{cases}
\frac{1}{\alpha(\alpha-1)}\sum^{M}_{i=1} (p_i - p^{\alpha}_i), \; &\alpha \neq 1, \\
-\sum^{M}_{i=1} (p_i - \log p_i), &\alpha = 1, 
\end{cases}
\end{align}
\begin{align}
\alpha\texttt{-Entmax}(\mbz) = \underset {\mbp\in \Delta^{M}}{\argmax} [\langle\mbp,\mbz\rangle-\Psi_{\alpha}(\mbp) ],
\end{align}
where $\mathbf{p} \in \mathbb{R}^{M}$.
Here, $\alpha$ controls the sparsity.
When $\alpha=1$, it is equivalent to a dense probability mapping, $1\texttt{-Entmax} = \texttt{Softmax}$, and as $\alpha$ increases towards $2$, the outputs of $\alpha\texttt{-Entmax}$ become increasingly sparse.
Similar to $\mt_{\texttt{Dense}}$, $\Tca$ can be extended to attention mechanisms, establishing a strong connection with sparse attention.
For $\alpha = 2$, the exact solution can be efficiently computed using a sorting algorithm~\cite{held1974validation,michelot1986finite}.
For $1 < \alpha < 2$, inaccurate and slow iterative algorithm was used for computing $\alpha\texttt{-Entmax}$~\cite{liu2009efficient}.
Interestingly, for $1.5\texttt{-Entmax}$, an exact solution are derived in a simple form~\cite{entmaxx}.  In SHN, sparsity reduces retrieval errors and provide faster convergeness compared to dense retrieval dynamics~\cite{sparsehop,stanhop}. 

As mentioned, the retrieval dynamics of modern and sparse Hopfield energy can be  converted into an attention mechanism as follows:  
\begin{align}
& \Att (\mathbf{Q}_t,\mathbf{K}_t,\mathbf{V}_t)  = \texttt{Softmax}(\mathbf{Q}_t\mathbf{K}_t^{\top}/\sqrt{d})\mathbf{V}_t \label{dense-atten} \\
 &\Att _{\alpha}(\alpha,\mathbf{Q}_t,\mathbf{K}_t,\mathbf{V}_t)  = \alpha\texttt{-Entmax}(\mathbf{Q}_t\mathbf{K}_t^{\top}/\sqrt{d})\mathbf{V}_t \label{spars-atten}
\end{align}
where $\texttt{At}^l$ denotes original (dense) attention layer, and $\texttt{At}_{\alpha}$ represents sparse attention module with $\ent$ operator at $l^{th}$ layer. Both attention layers can be applied to self and cross-attention layers.
$\mathbf{Q}_t$, $\mathbf{K}_t$,  and $\mathbf{V}_t$ represent the query, key, and value matrices at time step $t$, respectively, and $d$ is the dimensionality of the keys and queries.
Note that with $\beta = 1/\sqrt{d}$, weight matrices, and operators, $\mt_{\texttt{Dense}}$ in ~\cref{attention} and $\mt_{\alpha}$ in ~\cref{sparse-ret} are reduce to the transformer attention mechanism  ~\cref{dense-atten} and ~\cref{spars-atten}, respectively. More details are available in supplement~\ref{sec:background}.

\noindent\textbf{Noise robustness of sparse Hopfield network} While the sparse extension is an efficient counterpart of dense Hopfield network, it has been discovered that there is more advantages to use sparse one besides efficiency~\cite{sparsehop,stanhop}. 
\begin{thm}(Noise-Robustness)~\cite{sparsehop}.
In case of noisy patterns with noise $\boldsymbol{\eta}$, i.e. $\tilde{\mbx} = \mbx + \boldsymbol{\eta}$ (noise in query) or $\tilde{\mbxi}_{\mu} = \mbxi_{\mu} + \boldsymbol{\eta}$ (noise in memory), the impact of noise $\boldsymbol{\eta}$ on the sparse retrieval error $||\mt_{2}(\mbx) - \mbxi_\mu||$ is linear, while its effect on the dense retrieval error $||\mt_{\texttt{Dense}}(\mbx) - \mbxi_\mu||$ is exponential. \label{thm:1}
\end{thm}
\noindent where $\mbxi_{\mu}$ is memory pattern and to be considered stored at a fixed point of $\mt$. This theorem suggests that under noisy conditions, sparse attention mechanisms exhibit superior noise robustness compared to standard dense attention, leads the lower retrieval error.

\section{Main Contribution : PLADIS}
Motivated by advantages of sparse attention presented in previous section, we aimed to enhance text to image (T2I) diffusion model by sparsifying attention modules as described in \cref{spars-atten}. 
In the following sections, we investigate sparse attention in self- and cross-attention for T2I diffusion (Sec.~\ref{sec:compare_self_cross}), explore the effect of sparsity in $\ent$ for $\alpha > 1$ (Sec.~\ref{sec:effect_alpha}) and connect SHN’s noise robustness with sparse attention for $1 < \alpha \le2$ in T2I models (Sec.~\ref{sec:noise-robust}). Finally, we introduce PLADIS, a cost-effective enhancement method for T2I diffusion models (Sec.~\ref{sec:PLADIS}).



\begin{figure}[!t]
\centering
\includegraphics[width=0.95\linewidth]{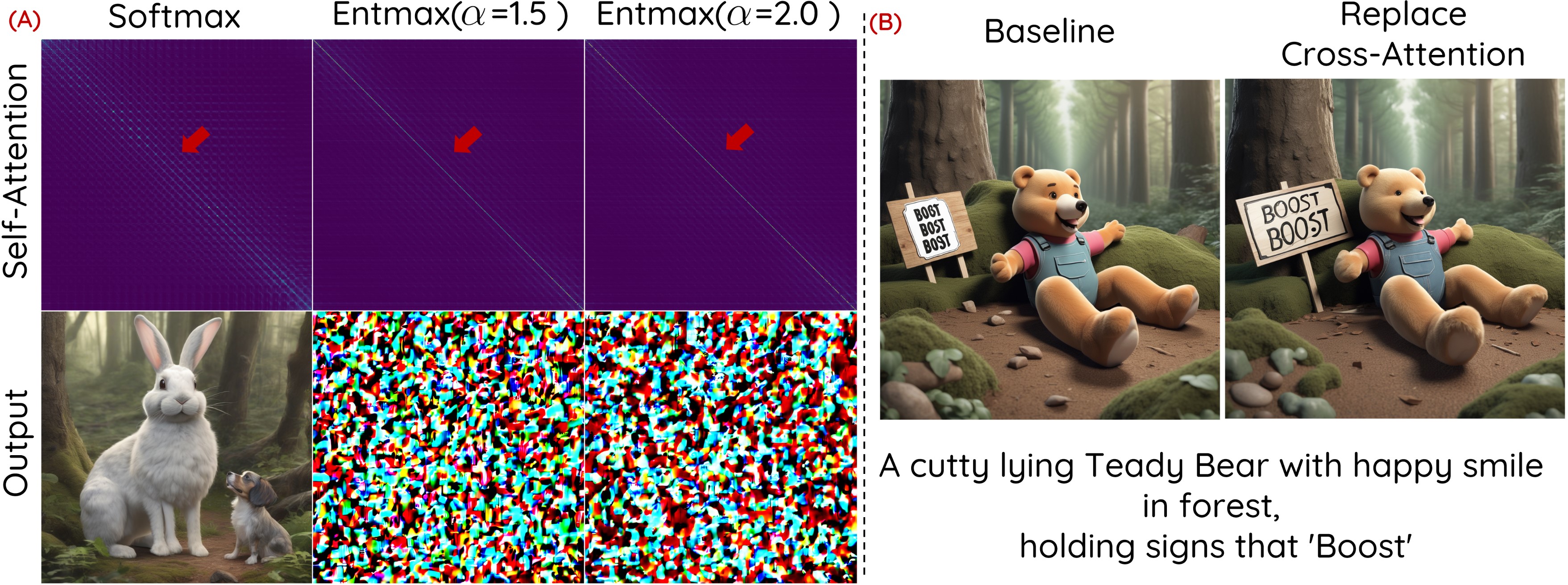}
\vspace{-0.5em}
\caption{Effectiveness of sparse attention mechanisms compared to baseline in (a) self- and (b) cross-attention variants.}
\vspace{-1.5em}
\label{fig:self_cross_compare}
\end{figure}

\subsection{Sparse Attention for T2I Generation}\label{sec:compare_self_cross}
To evaluate the efficacy of sparse attention mechanisms in text-to-image diffusion models, we first replace the standard self-attention with their sparse variants, specifically using $\ent$ with $\alpha=1.5$ and $2.0$ as shown Fig~\ref{fig:self_cross_compare} (a). Although 
Theorem~\ref{thm:1} applies to both self and cross-attention, we empirically observe that in the case of self-attention (i.e., image-to-image), most entries of $\alpha$-\texttt{Entmax}$(\mathbf{Q}\mathbf{K}^\top)$ are concentrated along the diagonal. This behavior causes each patch to attend primarily to itself, severely limiting inter-pixel interactions and ultimately leading to failure in image generation.
Surprisingly, substituting the cross-attention module with its sparse counterpart leads to enhanced generation quality and better text alignment, \emph{although the model was not trained with the sparse attention modules.}
As shown in Fig.~\ref{fig:self_cross_compare} (b), the baseline results are unable to accurately generate the text "Boost." In contrast, the sparse variants achieve successful and accurate text generation. Further evidence of these improvements can be found in Fig~\ref{fig:alpha}. This intriguing discovery regarding the use of sparse cross-attention within T2I diffusion models serves as the primary impetus behind our proposed algorithm.


\subsection{Effect of Sparsity in Cross-Attention Module }\label{sec:effect_alpha}
We further explore the effect of sparsity in the sparse attention mechanism within the cross-attention module of T2I diffusion models. 
Notably, $\ent$ transforms are sparse for all $\alpha > 1$. To assess sparsity’s impact, we replace standard cross-attention layers with sparse ones and generate 5K samples from the MS-COCO validation dataset using CFG guidance, varying $\alpha$ as shown in Fig.\ref{fig:alpha}. Interestingly, increasing sparsity (higher $\alpha$) improves generation quality, text alignment, and human preference scores without additional training.
Cross-attention with softmax results in dense alignments and strictly positive output probabilities, but sparse cross-attention produces sparse alignments, ensuring a stricter match between image and text embeddings. It leads to overall improvement in performance.

\subsection{Connection With Noise Robustness of SHN} \label{sec:noise-robust}
To further verify why performance improves when $1<\alpha \le2$, we introduce retrieval error of dynamics for this case:
\begin{thm}[Retrieval Error]\label{theorem-noise}
    Let $\Tca$ be the retrieval dynamics of Hopfield model with $\ent$. 
    \begin{align}
        \text{For } 1 < \alpha \le 2,
        &||\Tca(\xb)-\xib_\mu|| \le m +
        m \kappa \Big[ (\alpha-1)\beta \notag \\
        &\left( \max_\nu \langle \xib_\nu, \xb \rangle - \left[\Xib^\intercal \xb \right]_{(\kappa+1)} \right)\Big]^{\frac{1}{\alpha-1}},\label{eq:ineq-new}
    \end{align}
    Here, we abuse the notation $\left[\Xib^\intercal \xb \right]_{(d+1)}:= \left[\Xib^\intercal \xb \right]_{(d)} - M^{1-\alpha}/(\alpha-1)$.
\end{thm}
\noindent For proof, see supplement $\ref{sec:background}$.
Based on our proposed error bound, we can derive the noise-robustness for $1<\alpha \le 2$. 
\begin{corollary}(Noise-Robustness)
In case of noisy patterns with noise $\boldsymbol{\eta}$, the impact of noise on the retrieval error $||\Tca(\xb) - \xib_\mu||$ is polynomial of order $\frac{1}{\alpha-1}$ for $1<\alpha \le 2$. \label{collary-noise}
\end{corollary}
\noindent This theorem and corollary suggest that $\Tca$ also take pleasure in noise robustness for $1<\alpha \le2$, leads the lower retrieval error.  
In T2I diffusion models, cross-attention layers process query, key, and value matrices from noisy images and text prompts. Due to Gaussian noise corruption in the diffusion process, the query matrix is inherently perturbed. Building on this and Theorem~\ref{thm:1}, \ref{theorem-noise}, and Corollary~\ref{collary-noise}, the observed performance improvement, especially with increasing $\alpha$, reflects the noise robustness of sparse attention, as shown in Fig.~\ref{fig:alpha}. By linking these gains to the theoretical guarantees of SHN, we provide a stronger foundation for the efficacy of sparse-cross attention in DMs.

\begin{figure}[t!]
\centering
\includegraphics[width=1\linewidth]{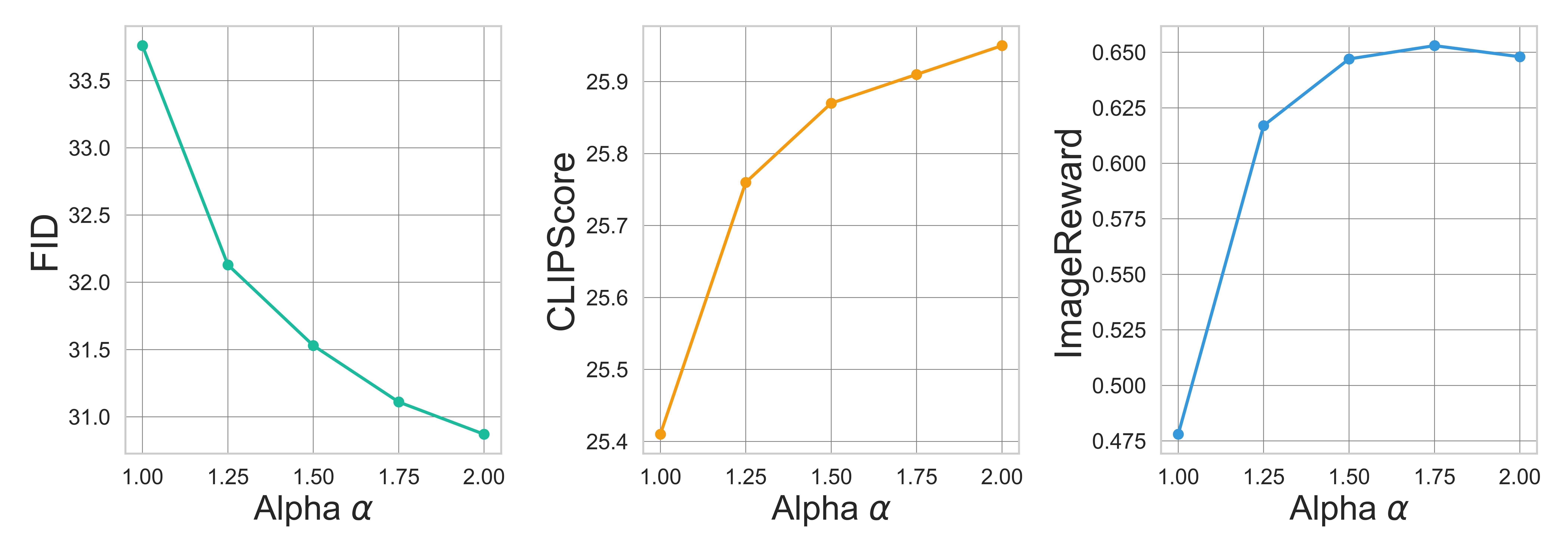}
\vspace{-2em}
\caption{Comparison of $\alpha$ values in $\ent$ on the MS-COCO dataset with CFG and PAG guidance. }\label{fig:alpha}
\vspace{-1em}
\end{figure}

\begin{algorithm}[t!]
	\caption{Diffusion Sampling with PLADIS and other guidance methods}
	\label{algo-PLADIS}
	\SetKwInOut{Input}{Input}
	\SetKwInput{kwInit}{return}
	\Input{Diffusion model $\me_{\theta}(\mbx_t)$ with cross-attention module $\texttt{At}(\cdot)$ at layer $l$, total number of cross-attention layers $L$, scales $\lambda$.}
    \For{$l$ in $1,, \cdots, L$ }{Replace $\texttt{At}(\cdot)$ with $\texttt{At}_{\texttt{Ours}}(\cdot)$ by Eq.~\ref{eq:PLADIS}}         
   $\mbx_{T} ~\sim \mathcal{N}(0,I)$\\
\For{$t$ in $T,T-1, \cdots, 1$ }{
\If{CFG}{Compute $\me_{\theta}(\mbx_t,\mbc)$ by Eq.~\ref{eq:cfg}}
\If{PAG or SEG}{Compute $\me_{\theta}(\mbx_t,\mbc)$ by Eq.~\ref{eq:weak}}
$\hat \mbx_0(t) = (\mbx_t - \sqrt{1- \bar{\alpha_t}}\me_{\theta}(\mbx_t,\mbc)) / \sqrt{\bar{\alpha_t}}$ \\
$\mbx_{t-1} = \sqrt{\bar{\alpha}_{t-1}} \hat \mbx_0(t)  + \sqrt{1- \bar{\alpha}_{t-1}}\me_{\theta}(\mbx_t,\mbc)$}
\kwInit{$\mbx_{0}$}
\end{algorithm}

\subsection{Our Approach : PLADIS}\label{sec:PLADIS}
Building on our exploration of sparse attention, we propose a simple yet more effective approach called PLADIS. Specifically, we aim to enhance the benefits of sparse attention (as shown in Fig.~\ref{fig:alpha}) without introducing additional neural function evaluations (NFEs). Inspired by guidance methods like CFG, PAG, and SEG, we extrapolate query-key correlations in both dense and sparse attentions.
\begin{align}
\texttt{At}_{\texttt{Ours}}(\alpha,\lambda, \mathbf{Q}_t,\mathbf{K}_t,\mathbf{V}_t) &:=      \texttt{At}(\mathbf{Q}_t,\mathbf{K}_t,\mathbf{V}_t) \;+ \nonumber\\
\lambda\big(\texttt{At}_{\alpha}(\alpha,\mathbf{Q}_t,\mathbf{K}_t,\mathbf{V}_t) &-\texttt{At}(\mathbf{Q}_t,\mathbf{K}_t,\mathbf{V}_t)\big)\label{eq:PLADIS}
\end{align}
The scale parameters $\lambda$ is a hyperparameter and determine the extent to which sparse attention effects are accentuated. 
When $\lambda = 0$, the formula is equivalent to the baseline model, and when $\lambda = 1$, it represents the model in \cref{sec:effect_alpha}. When $\lambda >1$, our PLADIS is applied.
The sparsity degree $1 < \alpha \le 2$ is another hyperparameter, but we only consider two options $\alpha = 1.5$ and $\alpha = 2$ , where efficient algorithms are known to exist.

Here, we emphasize the generalizability of our method.
Other methods that modify the attention module require hyperparameter search for target layers.
However, for PLADIS, applying \cref{eq:PLADIS} to all cross-attention layers is sufficient, which makes our method more easily extendable to other cases.
Nevertheless, we conduct an ablation study in \cref{tab_layer} for varying target layers and find that applying it to all layers is the optimal choice. (See supplement~\ref{sec:add_ablation}) 
Moreover, unlike other guidance methods, our method is implicit in that it does not require an additional step, enabling ours to be extended to guidance-distilled models.


\section{Experiment}
\noindent\textbf{Implementation Detail}
In our experiments, we use Stable Diffusion XL (SDXL) \cite{sdxl} as the backbone model to validate the effectiveness of our proposed methods. The results on other backbone is available in supplement \ref{sec:add_other}. All experiments are conducted on a single NVIDIA H100 GPU. For the calculation of the $\alpha\texttt{-Entmax}$ function, we utilize an open-source library\footnote{\url{https://github.com/deep-spin/entmax}}.
We set $\alpha$ to 1.5 and the scale $\lambda$ to 2.0 as the baseline. 

\noindent\textbf{Evaluation Metric}
To comprehensively assess our method, we employ various evaluation metrics. For visual fidelity, we calculate the Frechet Inception Distance (FID)~\cite{fid} of images generated from 30K random prompts from the MS-COCO validation set~\cite{coco}. To evaluate text-image alignment and user preference, we measure CLIPScore~\cite{clipscore}, ImageReward~\cite{imagereward}, PickScore~\cite{pick}, and Human Preference Score (HPS v2.1)\cite{hpsv2}. Additionally, our model is evaluated using text prompts from not only MS-COCO but also Drawbench~\cite{drawbench}, HPD~\cite{hpsv2}, and Pick-a-Pic~\cite{pick}. More details are provided in the supplement~\ref{sec:metric_detail}.

\begin{table}[t!]
\caption{Quantitative results of various guidance methods on the MS-COCO dataset. Bold text indicates the best performance for each metric across the different methods.}
\vspace{-0.5em}
\label{tab_main}
\centering
\resizebox{0.9\linewidth}{!}{
\begin{small}
\begin{tabular}{clccc}
\toprule
CFG & Method & \multicolumn{1}{c}{FID~$\downarrow$} & \multicolumn{1}{c}{CLIPScore~$\uparrow$} & \multicolumn{1}{c}{ImageReward~$\uparrow$} \\
\cmidrule(r){1-1}\cmidrule(r){2-2}\cmidrule(r){3-5}
\multirow{6}{*}{\xmark} & Vanilla & 83.68	& 20.92	& -1.050  \\
\cmidrule(r){2-2}\cmidrule(r){3-5}
 \rowcolor{green!10}  \cellcolor{white} & +\,Ours & \bf 79.72 \scriptsize{\textcolor{blue}{(-3.96)}}	&\bf 21.86\scriptsize{\textcolor{blue}{(+0.89)}}	&\bf -0.858	\scriptsize{\textcolor{blue}{(+0.19)}} \\
\cmidrule(r){2-5}
 & PAG~\cite{PAG} & 29.36	& 24.03	& -0.011 \\
 \cmidrule(r){2-2}\cmidrule(r){3-5}
 \rowcolor{green!10}  \cellcolor{white} &+\,Ours&\bf 24.51 \scriptsize{\textcolor{blue}{(-4.85)}}	&\bf24.85 \scriptsize{\textcolor{blue}{(+0.93)}}	&\bf0.251 \scriptsize{\textcolor{blue}{(+0.31)}} \\
\cmidrule(r){2-5}
 & SEG~\cite{SEG} & 38.08	& 23.71	&-0.139	 \\
 \cmidrule(r){2-2}\cmidrule(r){3-5}
 \rowcolor{green!10}  \cellcolor{white} &+\,Ours &\bf 33.19 \scriptsize{\textcolor{blue}{(-4.89)}}	&\bf24.63	\scriptsize{\textcolor{blue}{(+1.02)}} &\bf0.134 \scriptsize{\textcolor{blue}{(+0.28)}}\\
\midrule
\multirow{6}{*}{\cmark} & Vanilla & 23.39	&25.91	&0.425 \\
\cmidrule(r){2-2}\cmidrule(r){3-5}
 \rowcolor{green!10}  \cellcolor{white} &+\,Ours &\bf 19.01 \scriptsize{\textcolor{blue}{(-4.38)}}	&\bf26.61 \scriptsize{\textcolor{blue}{(+0.70)}}	&\bf0.622 \scriptsize{\textcolor{blue}{(+0.20)}} \\
\cmidrule(r){2-5}
 & PAG~\cite{PAG} & 24.32	&25.42	& 0.478 \\
 \cmidrule(r){2-2}\cmidrule(r){3-5}
 \rowcolor{green!10}  \cellcolor{white} &+\,Ours &\bf20.11	\scriptsize{\textcolor{blue}{(-4.21)}}&\bf26.41 \scriptsize{\textcolor{blue}{(+0.99)}}	&\bf0.726 \scriptsize{\textcolor{blue}{(+0.25)}}\\
\cmidrule(r){2-5}
 & SEG~\cite{SEG} & 26.80	&25.39	& 0.431	\\
 \cmidrule(r){2-2}\cmidrule(r){3-5}
 \rowcolor{green!10}  \cellcolor{white} &+\,Ours&\bf22.08 \scriptsize{\textcolor{blue}{(-4.80)}}	&\bf26.49 \scriptsize{\textcolor{blue}{(+1.10)}}	&\bf0.689	\scriptsize{\textcolor{blue}{(+0.26)}}\\       
\bottomrule
\end{tabular}
\end{small}
}
\vspace{-0.5em}
\end{table}

\begin{table}[t!]
\caption{Quantitative comparison of text alignment and human preference across datasets using various guidance methods. For PAG, SEG, CFG guidance is used jointly. Bold text indicates the best performance for each metric.}
\label{tab_hps}
\vspace{-0.5em}
\centering
\resizebox{0.9\linewidth}{!}{
\begin{small}
\begin{tabular}{clcccc}
\toprule
Dataset & Method & CLIPScore~$\uparrow$ & PickScore~$\uparrow$ & ImageReward~$\uparrow$ & HPSv2~$\uparrow$ \\
\cmidrule(r){1-1}\cmidrule(r){2-2}\cmidrule(r){3-6}
\multirow{9}{*}{Drawbench~\cite{drawbench}} & CFG~\cite{CFG} & 26.63 &21.72 & 0.198 & 26.83 \\
\cmidrule(r){2-2}\cmidrule(r){3-6}
& $\cellcolor{green!10}$ +\,Ours & \cellcolor{green!10} \bf 27.72 \scriptsize{\textcolor{blue}{(+1.09)}} & \cellcolor{green!10} \bf 21.94 \scriptsize{\textcolor{blue}{(+0.22)}} & \cellcolor{green!10} \bf0.419 \scriptsize{\textcolor{blue}{(+0.22)}} & \cellcolor{green!10}\bf27.10 \scriptsize{\textcolor{blue}{(+0.24)}} \\
\cmidrule(r){2-6}
 & PAG~\cite{PAG} & 26.19 & 21.94 & 0.295 & 28.65 \\
\cmidrule(r){2-2}\cmidrule(r){3-6}
& \cellcolor{green!10}+\,Ours & \cellcolor{green!10} \bf27.23 \scriptsize{\textcolor{blue}{(+1.05)}} & \cellcolor{green!10} \bf22.16 \scriptsize{\textcolor{blue}{(+0.22)}} & \cellcolor{green!10} \bf0.570 \scriptsize{\textcolor{blue}{(+0.27)}} & \cellcolor{green!10} \bf28.93 \scriptsize{\textcolor{blue}{(+0.28)}} \\
\cmidrule(r){2-6}
 & SEG~\cite{SEG} & 26.06 & 21.79 & 0.291 & 28.71 \\
\cmidrule(r){2-2}\cmidrule(r){3-6}
\rowcolor{green!10}  \cellcolor{white} & +\,Ours & \bf27.41 \scriptsize{\textcolor{blue}{(+1.34)}} & \bf21.99 \scriptsize{\textcolor{blue}{(+0.20)}} & \bf0.497 \scriptsize{\textcolor{blue}{(+0.21)}} & \bf29.08 \scriptsize{\textcolor{blue}{(+0.37)}}\\
\midrule
\multirow{9}{*}{HPD~\cite{hpsv2}} & CFG~\cite{CFG} & 29.00 & 21.98 & 0.567 & 28.53 \\
\cmidrule(r){2-2}\cmidrule(r){3-6}
\rowcolor{green!10}  \cellcolor{white} &+\,Ours & \bf29.78 \scriptsize{\textcolor{blue}{(+0.78)}}  & \bf22.11 \scriptsize{\textcolor{blue}{(+0.13)}}  & \bf0.693 \scriptsize{\textcolor{blue}{(+0.13)}}  & \bf28.54 \scriptsize{\textcolor{blue}{(+0.01)}} \\
\cmidrule(r){2-6}
 & PAG~\cite{PAG} & 28.01 & 22.13 & 0.637 & 30.64 \\
\cmidrule(r){2-2}\cmidrule(r){3-6}
 & \cellcolor{green!10} +\,Ours & \cellcolor{green!10} \bf28.93 \scriptsize{\textcolor{blue}{(+0.92)}}  & \cellcolor{green!10}  \bf22.35 \scriptsize{\textcolor{blue}{(+0.22)}} & \cellcolor{green!10} \bf0.828 \scriptsize{\textcolor{blue}{(+0.19)}}  & \cellcolor{green!10} \bf31.12 \scriptsize{\textcolor{blue}{(+0.48)}}  \\
\cmidrule(r){2-6}
 & SEG~\cite{SEG} & 28.21 & 21.98 & 0.673 & 30.48 \\
\cmidrule(r){2-2}\cmidrule(r){3-6}
\rowcolor{green!10}  \cellcolor{white} &+\,Ours & \bf 29.21 \scriptsize{\textcolor{blue}{(+1.00)}}  & \bf22.15 \scriptsize{\textcolor{blue}{(+0.17)}}  & \bf 0.786 \scriptsize{\textcolor{blue}{(+0.11)}}  & \bf 30.75 \scriptsize{\textcolor{blue}{(+0.27)}} \\
\midrule
\multirow{9}{*}{Pick-a-pic~\cite{pick}} & CFG~\cite{CFG} & 27.08 & 21.30 & 0.340 & 28.05 \\
\cmidrule(r){2-2}\cmidrule(r){3-6}
\rowcolor{green!10}  \cellcolor{white} &+\,Ours & \bf27.97 \scriptsize{\textcolor{blue}{(+0.89)}} & \bf21.69 \scriptsize{\textcolor{blue}{(+0.09)}}  & \bf0.466 \scriptsize{\textcolor{blue}{(+0.13)}}  & \bf28.14 \scriptsize{\textcolor{blue}{(+0.09)}}  \\
\cmidrule(r){2-6}
 & PAG~\cite{PAG}& 26.34 & 21.49 & 0.467 & 29.91 \\
\cmidrule(r){2-2}\cmidrule(r){3-6}
  & \cellcolor{green!10} +\,Ours & \cellcolor{green!10} \bf27.31 \scriptsize{\textcolor{blue}{(+0.97)}}  & \cellcolor{green!10} \bf21.67 \scriptsize{\textcolor{blue}{(+0.18)}}  & \cellcolor{green!10} \bf0.668 
 \scriptsize{\textcolor{blue}{(+0.20)}} &\cellcolor{green!10}  \bf30.38 \scriptsize{\textcolor{blue}{(+0.47)}}  \\
\cmidrule(r){2-6}
 & SEG~\cite{SEG} & 26.48 & 21.36 & 0.461 & 29.38 \\
\cmidrule(r){2-2}\cmidrule(r){3-6}
\rowcolor{green!10}  \cellcolor{white} &+\,Ours & \bf27.50 \scriptsize{\textcolor{blue}{(+1.02)}}  & \bf21.48 \scriptsize{\textcolor{blue}{(+0.12)}}  & \bf0.613 \scriptsize{\textcolor{blue}{(+0.15)}}  & \bf30.15 \scriptsize{\textcolor{blue}{(+0.77)}} \\
\bottomrule	        
\end{tabular}
\end{small}
}
\vspace{-1em}
\end{table}

\begin{table*}[t!]
\caption{Quantitative comparison across various datasets using 4-steps sampling with the guidance-distilled model.}
\vspace{-0.5em}
\label{tab_distill}
\centering
\resizebox{0.80\linewidth}{!}{
\begin{small}
\begin{tabular}{lccccccccc}
\toprule
 & \multicolumn{3}{c}{Drawbench~\cite{drawbench}} & \multicolumn{3}{c}{HPD~\cite{hpsv2}} & \multicolumn{3}{c}{Pick-a-pic~\cite{pick}} \\
\cmidrule(r){2-4}\cmidrule(r){5-7}\cmidrule(r){8-10}
Method &  CLIPScore~$\uparrow$ & PickScore~$\uparrow$ & ImageReward~$\uparrow$  &  CLIPScore~$\uparrow$ & PickScore~$\uparrow$ & ImageReward~$\uparrow$  &  CLIPScore~$\uparrow$ & PickScore~$\uparrow$ & ImageReward~$\uparrow$ \\
\cmidrule(r){1-1}\cmidrule(r){2-4}\cmidrule(r){5-7}\cmidrule(r){8-10}
Turbo~\cite{turbo} & 27.81 &22.11 &	0.555	& 29.06 &	22.39	& 0.733	& 27.41	& 21.75	& 0.625 \\
\cmidrule(r){1-1}\cmidrule(r){2-4}\cmidrule(r){5-7}\cmidrule(r){8-10}
\rowcolor{green!10}+\,Ours & \bf{28.55} \scriptsize{\textcolor{blue}{(+0.73)}}& \bf{22.18} \scriptsize{\textcolor{blue}{(+0.07)}}&	\bf{0.601} \scriptsize{\textcolor{blue}{(+0.05)}}&	\bf{29.56} \scriptsize{\textcolor{blue}{(+0.50)}}& \bf{22.44} \scriptsize{\textcolor{blue}{(+0.05)}}& \bf{0.754} \scriptsize{\textcolor{blue}{(+0.02)}}&\bf{27.92} \scriptsize{\textcolor{blue}{(+0.52)}}& \bf{21.77} \scriptsize{\textcolor{blue}{(+0.02)}}&	\bf{0.657} \scriptsize{\textcolor{blue}{(+0.03)}}
 \\
\midrule
Light~\cite{light} & 26.86	&22.30	&0.625	&28.77&	22.70&	0.931&	27.19&	22.03&	0.827
 \\
\cmidrule(r){1-1}\cmidrule(r){2-4}\cmidrule(r){5-7}\cmidrule(r){8-10}
\rowcolor{green!10}+\,Ours & \bf27.70 \scriptsize{\textcolor{blue}{(+0.84)}}&	\bf22.39 \scriptsize{\textcolor{blue}{(+0.09)}}&	\bf0.738 \scriptsize{\textcolor{blue}{(+0.11)}} &\bf	29.41 \scriptsize{\textcolor{blue}{(+0.64)}}&\bf	22.76 \scriptsize{\textcolor{blue}{(+0.06)}} & \bf1.011 \scriptsize{\textcolor{blue}{(+0.08)}} & \bf27.91 \scriptsize{\textcolor{blue}{(+0.72)}}&	\bf22.09	\scriptsize{\textcolor{blue}{(+0.06)}}&\bf0.891 \scriptsize{\textcolor{blue}{(+0.07)}}
 \\
\midrule
DMD2~\cite{dmd2} & 28.08 &	22.39&	0.829&	29.78	&22.55&	1.002&	28.14 &	21.88&	0.983
 \\
\cmidrule(r){1-1}\cmidrule(r){2-4}\cmidrule(r){5-7}\cmidrule(r){8-10}
\rowcolor{green!10}+\,Ours & \bf28.38 \scriptsize{\textcolor{blue}{(+0.30)}}&\bf	22.41 \scriptsize{\textcolor{blue}{(+0.02)}}&\bf0.919 
\scriptsize{\textcolor{blue}{(+0.09)}}&\bf	29.94 	\scriptsize{\textcolor{blue}{(+0.16)}}&\bf22.60 \scriptsize{\textcolor{blue}{(+0.05)}}&\bf	1.043 \scriptsize{\textcolor{blue}{(+0.04)}}&	\bf28.53  \scriptsize{\textcolor{blue}{(+0.39)}}&	\bf21.91 \scriptsize{\textcolor{blue}{(+0.03)}}&\bf	0.993 \scriptsize{\textcolor{blue}{(+0.01)}}
 \\
\midrule
Hyper~\cite{hyper} & 27.51&	22.53&	0.768&	29.27	&22.86	&1.123	&27.63	&22.15&	1.023
 \\
\cmidrule(r){1-1}\cmidrule(r){2-4}\cmidrule(r){5-7}\cmidrule(r){8-10}
\rowcolor{green!10}+\,Ours & \bf28.22 \scriptsize{\textcolor{blue}{(+0.71)}}&\bf	22.60 \scriptsize{\textcolor{blue}{(+0.07)}}&	\bf0.867 \scriptsize{\textcolor{blue}{(+0.10)}}&\bf	29.80 \scriptsize{\textcolor{blue}{(+0.53)}}	&\bf22.96 \scriptsize{\textcolor{blue}{(+0.10)}}&	\bf1.184 \scriptsize{\textcolor{blue}{(+0.06)}}&	\bf28.27 \scriptsize{\textcolor{blue}{(+0.64)}}&	\bf22.23 \scriptsize{\textcolor{blue}{(+0.08)}}&\bf	1.111 \scriptsize{\textcolor{blue}{(+0.09)}}
 \\
\bottomrule	        
\end{tabular}
\end{small}
}
\vspace{-0.5em}
\end{table*}

\begin{figure*}[t!]
\centering
\includegraphics[width=0.80\linewidth]{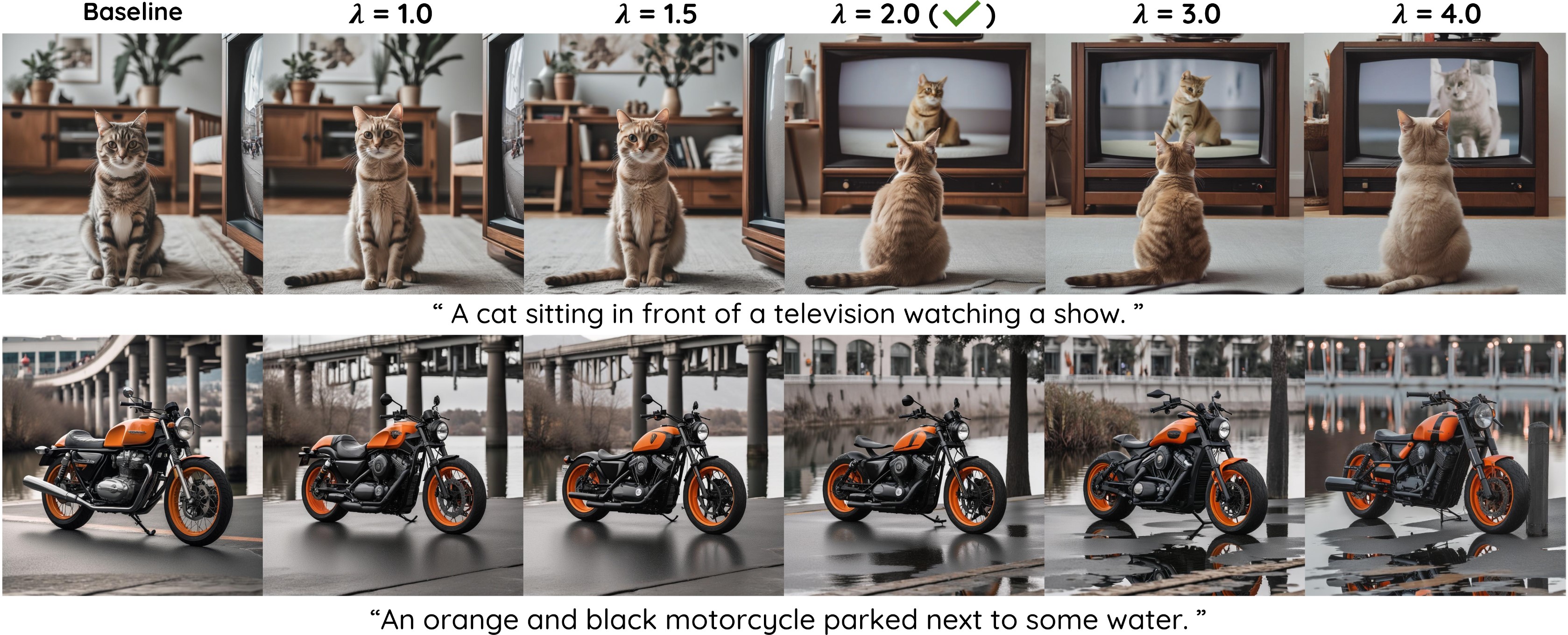}
\vspace{-0.5em}
\caption{Qualitative comparison by varying the scale $\lambda$. 
 As the scale $\lambda$ increases, images represent improved plausibility and enhanced text alignment. But too high a value leads to smoother textures and potential artifacts, similar to those seen in CFG. When $\lambda$ is greater than 0, our PLADIS method is applied. In our configuration, $\lambda$ is set to 2.0.}\label{fig:lambda}
 \vspace{-1.5em}
\end{figure*}

\section{Results}
\noindent\textbf{Results with Guidance Sampling}\label{sec:guidance}
To rigorously evaluate the effectiveness of our method, we generate 30K samples both with and without CFG, applying various guidance sampling techniques, including PAG~\cite{PAG} and SEG~\cite{SEG}. In this setup, we use 25 sampling steps, and detail setting are available in supplement~\ref{sec:metric_detail}. As shown in Tab.~\ref{tab_main}, the use of PLADIS without any additional guidance sampling noticeably enhances visual quality, text alignment, and user preference. Furthermore, our method integrates seamlessly with different guidance approaches, offering straightforward yet impactful improvements when CFG and weak model guidance are used together. To further substantiate these findings, we conducted experiments on a human preference dataset, as illustrated in Tab.~\ref{tab_hps}. Our analysis reveals that ours consistently delivers substantial performance gains across all metrics and guidance techniques.
 Furthermore, the synergy between our method and existing guidance methods results in more visually appealing outputs and improved text-image coherence, as shown in Fig.~\ref{fig:main} and ~\ref{fig:lambda}. Further comparisons are provided in supplement~\ref{sec:add_example}. 
Notably, the combination of PLADIS with CFG and PAG provide superior results, establishing itself as a leading candidate among guidance approaches.

\noindent\textbf{Unleashing restrained concepts} In \cref{fig:lambda}, the baseline model does not produce the concepts correctly. It initially appears that the concept (spatial relation) is difficult for the model to learn and that a superior model is required to generate such concepts. However, the model already possesses knowledge of the relation; it merely fails to fully utilize its learned information. All we need is modifying inference steps to enable utilization, effectively surfacing the model's pre-existing knowledge and allowing it to fully realize and express previously latent concepts.

\noindent\textbf{Results on Guidance-Distilled Model} \label{sec:distill}
To validate the effectiveness of our method on the guidance-distilled model, we conduct experiments using various baselines with 4-steps sampling across different datasets, as shown in Tab.~\ref{tab_distill}. For the baselines, we employ several state-of-the-art methods, including SDXL-Turbo~\cite{turbo}, SDXL-Lighting (Light)~\cite{light}, Distribution Matching Distillation 2 (DMD2)~\cite{dmd2}, and Hyper-SDXL~\cite{hyper}. Notably, our method significantly enhances overall performance, particularly in terms of text alignment and human preference, across all baselines. The introduction of PLADIS improves the visual quality of samples compared to those produced by the baselines, as shown in Fig.~\ref{fig:main}. Furthermore, we observe that PLADIS also improves performance in one-step sampling. Due to space limitations, further examples and details are provided in the supplement~\ref{sec:add_onestep} and ~\ref{sec:add_example}.

\noindent\textbf{User Preference Study} 
Beyond the automated metrics, we aim to assess the practical effectiveness of PLADIS in terms of sample quality and prompt alignment. To evaluate human preference in these aspects, we have evaluators assess pairwise outputs from the model with and without PLADIS, associated with two questions. Fig. \ref{fig:user_study} presents the user study results. Notably, all guidance methods and distilled models with ours outperform those without ours in both image quality and prompt alignment. Especially, the models with ours significantly improve prompt coherence. Further details of the user preference study are available in supplement \ref{sec:user_study}.

\begin{figure}[t!]
\centering
\includegraphics[width=0.9\linewidth]{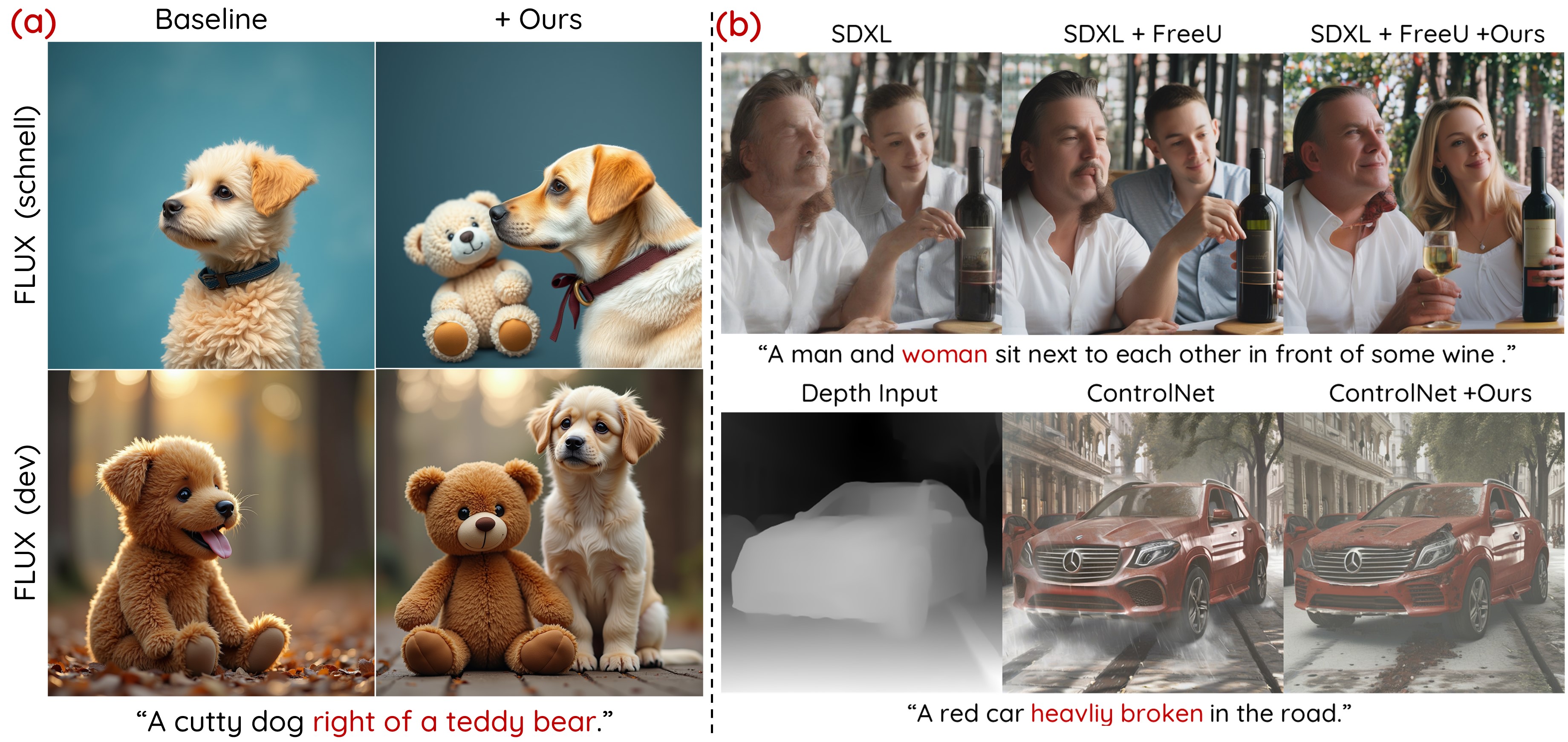}
\vspace{-1em}
\caption{(a) Comparison with and without our method in Flux. (b) Comparison results in FreeU and ControlNet.}\label{fig:exten}
\vspace{-1em}
\end{figure}


\noindent\textbf{Extension to Broader Frameworks and Backbones.}
We validate the generalizability of PLADIS across diverse backbones and inference settings. On the MMDiT backbone~\cite{stable3}, including Flux-Schnell and dev variants~\cite{flux2024}, PLADIS achieves notable gains on the Geneval benchmark~\cite{geneval}, which evaluates both visual quality and prompt alignment (Tab~\ref{tab:geneval}, Fig~\ref{fig:exten}(a)). PLADIS also complements FreeU~\cite{freeu}, enhancing fidelity and coherence when combined (Fig~\ref{fig:exten}(b, top); see also Supp. Tab~\ref{tab_another}).

Furthermore, we apply PLADIS to ControlNet~\cite{controlnet} for structure-guided generation. As shown in Fig~\ref{fig:exten}(b, bottom), PLADIS improves semantic accuracy, especially for complex prompts (e.g., “heavily broken red car”) that vanilla ControlNet tends to oversimplify. These results confirm PLADIS’s broad compatibility and effectiveness across models and inference-time methods.


\begin{table}[t!]
\caption{Ablation study on the $\alpha$ scale for $\alpha$$\texttt{-Entmax}$ with 25 steps. Inference time is measured per prompt.}
\vspace{-0.5em}
\label{tab_alpha}
\centering
\resizebox{0.93\linewidth}{!}{
\begin{small}
\begin{tabular}{cc|cccc|l|l}
\toprule
 $\alpha$ & 1  & 1.25 & 1.5 & 1.75 & 2 &  Ours($\alpha$ = 1.5) & Ours($\alpha$ = 2) \\
\cmidrule(r){1-1}\cmidrule(r){2-8}
FID $\downarrow$ & 33.76 & 32.13 & 31.53 & 31.11 & 30.87 & {27.87} \scriptsize{\textcolor{blue}{(-5.89)}} & {26.88} \scriptsize{\textcolor{blue}{(-6.88)}}  \\
CLIPScore $\uparrow$ & 25.41 & 25.76 & 25.87 & 25.91 & 25.95 & {26.41} \scriptsize{\textcolor{blue}{(+1.00)}} & {26.56} \scriptsize{\textcolor{blue}{(+1.15)}} \\
ImageReward $\uparrow$ & 0.478 & 0.617 & 0.647 & 0.653 & 0.648 & {0.726} \scriptsize{\textcolor{blue}{(+0.25)}} & {0.649} \scriptsize{\textcolor{blue}{(+0.001)}} \\
\cmidrule(r){1-1}\cmidrule(r){2-8}
Inference Time (sec) $\downarrow$ & \bf{2.521} & 9.172 & 3.085 & 9.097 & 2.785 & 3.087 \scriptsize{\textcolor{red}{(+0.56)}}  & 2.788 \scriptsize{\textcolor{red}{(+0.28)}}\\
Memory (G) $\downarrow$ & \bf{16.44} & 16.56 & 16.45 & 16.56 & 16.45 & 16.45 \scriptsize{\textcolor{red}{(+0.01)}} & 16.45 \scriptsize{\textcolor{red}{(+0.01)}}\\
\bottomrule	        
\end{tabular}
\end{small}
}
\vspace{-0.5em}
\end{table}

\begin{table}[!t]
\centering
\begin{minipage}[t]{0.48\linewidth}
\caption{Quantative comparison on Geneval.} 
\vspace{-0.5em}
\centering
\label{tab:geneval}
\resizebox{0.9\linewidth}{!}{
\begin{small}
\begin{tabular}{l c}
\toprule
Method & Overall Score \\
\midrule
FLUX (schnell)        & 0.671 \\
\cmidrule(r){1-1}\cmidrule(r){2-2}
\rowcolor{green!10} + Ours  & \bf 0.713  \\
\cmidrule(r){1-1}\cmidrule(r){2-2}
FLUX (dev)        & 0.666 \\
\cmidrule(r){1-1}\cmidrule(r){2-2}
\rowcolor{green!10} + Ours  & \bf 0.691 \\
\bottomrule
\end{tabular}
\end{small}
}
\end{minipage}
\hfill
\begin{minipage}[t]{0.5\linewidth}
\captionof{figure}{User Preference Study.}\label{fig:user_study}
\vspace{-1em}
\includegraphics[width=0.8\linewidth]{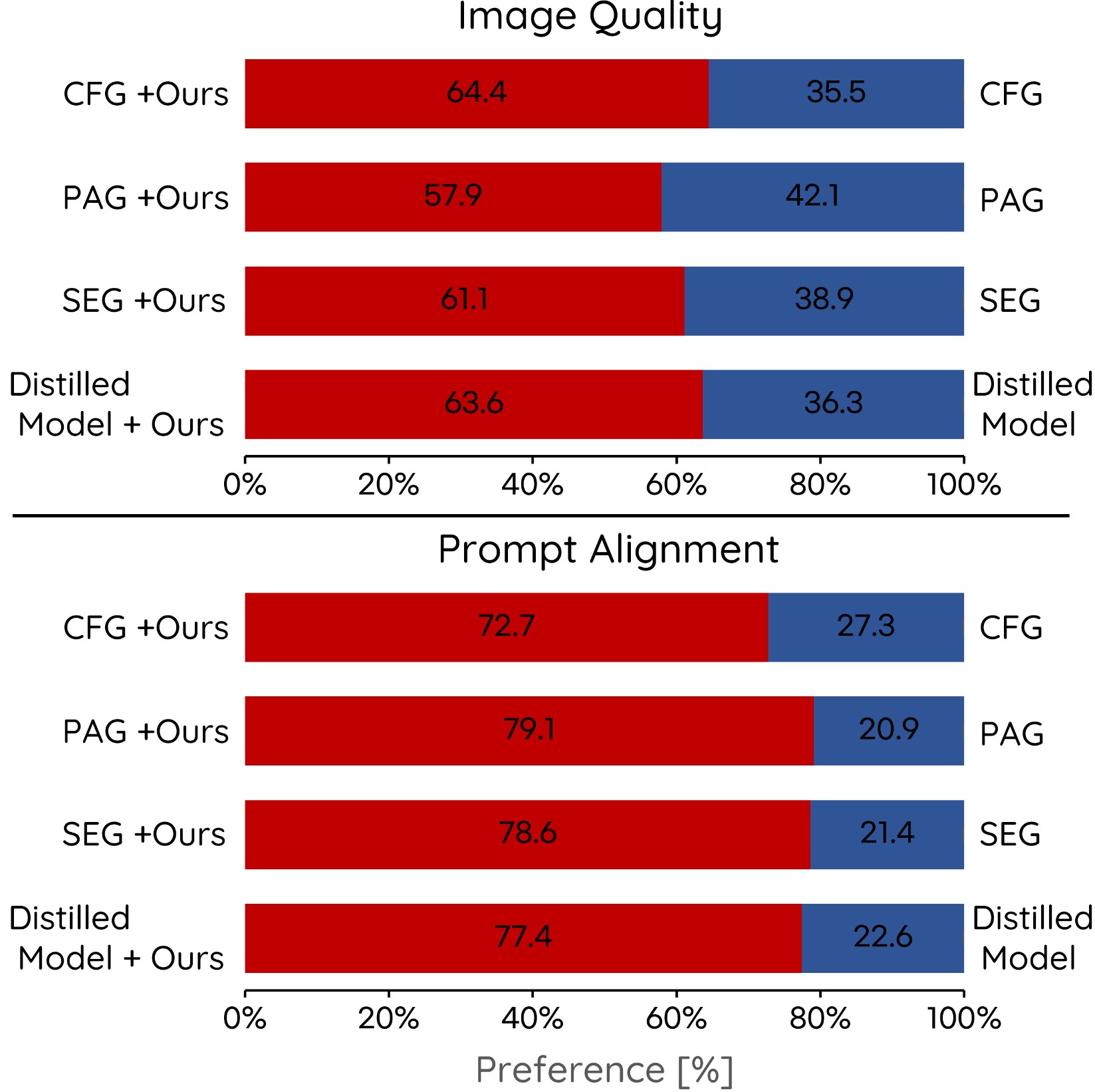}
\vspace{-1em}
\end{minipage}
\vspace{-0.5em}
\end{table}

\section{Ablation Study and Analysis}\label{ablation}
\noindent\textbf{The Effect of $\alpha$} We investigate the impact of $\alpha$ by adjusting its value in $\ent$, as shown in Fig.~\ref{fig:alpha} and Tab.~\ref{tab_alpha}. We generate 5K samples using CFG and PAG guidance on MS-COCO dataset. When $\alpha = 1$, this corresponds to baseline sampling with the $\texttt{Softmax}$ operation. For $\alpha > 1$, the cross-attention mechanism is replaced with the corresponding operation in $\ent$. Notably, introducing sparsity into cross-attention consistently enhances performance across all instances for $\alpha > 1$, supporting our theoretical findings on noise robustness of sparse attention in diffusion. 
In PLADIS, $\alpha$ values such as 1.5 and 2 are considered candidates.
Our approach ($\alpha$ = 2) provides the best performance in terms of FID and CLIPScore but obtains inferior results for ImageReward.
An $\alpha$ value of 1.5 offers balanced improvements across all metrics, making it our default setting.

\begin{figure}[t!]
\centering
\includegraphics[width=0.93\linewidth]{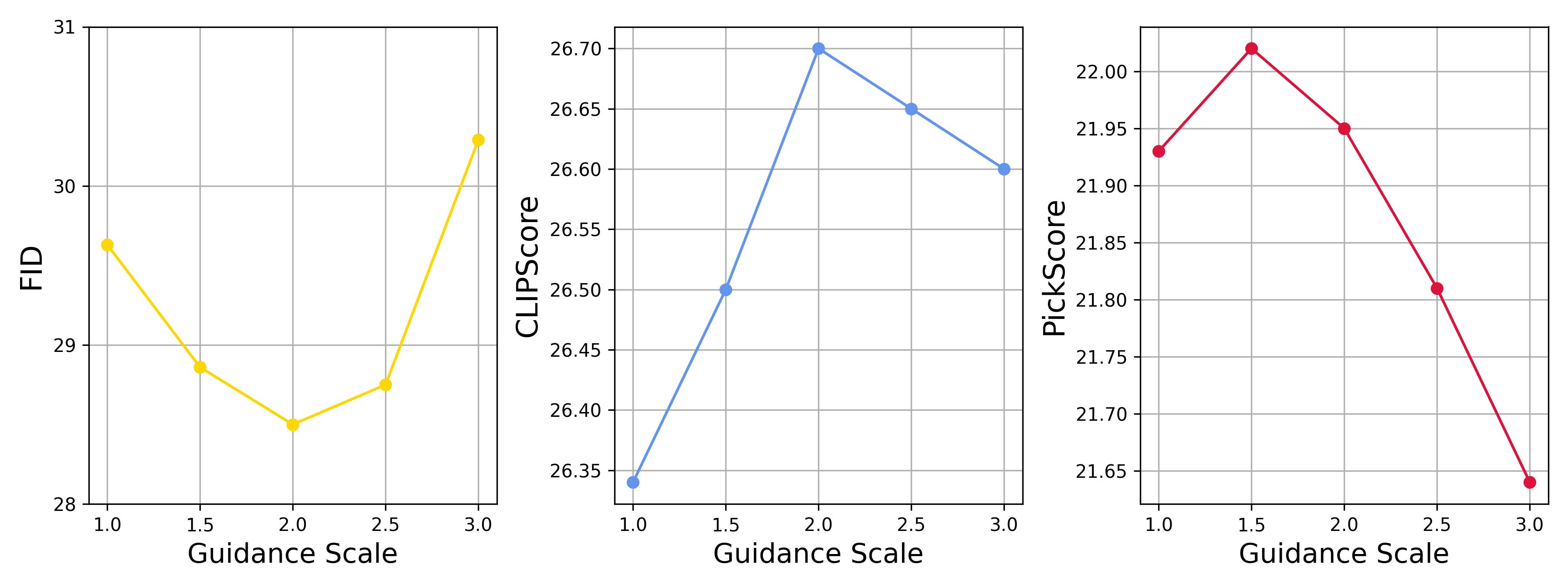}
\vspace{-1em}
\caption{Ablation study on the scale, $\lambda$, for PLADIS.}\label{fig:ablation}
\vspace{-1.5em}
\end{figure}

\noindent\textbf{Computation Cost} To evaluate the efficiency of PLADIS, we compare inference time and memory usage in VLAM by varying $\alpha$, as shown in Tab.~\ref{tab_alpha}. Unlike other guidance techniques, our PLADIS does not need extra inference at each time step, though it does involve calculating $\ent$ systematically. We observe that our method delivers the best performance while sacrificing minor processing time per prompt (0.56 seconds) and memory consumption (0.01 GB) compared to the baseline. Notably, our default setting ($\alpha$=1.5) is approximately 3$\times$ faster than other $\alpha$ values, except for $\alpha = 2$, and shows negligible differences compared to $\alpha = 1.5$ without PLADIS.

\noindent\textbf{The Scale $\lambda$} The scale $\lambda$ controls how much sparse attention with $\ent$ deviates from dense attention. A higher scale increases the influence of sparse attention relative to dense attention during denoising. In our empirical study, we sample 5K images with scales from 1.0 to 3.0, evaluating results using FID, CLIPScore, and PickScore (Fig.~\ref{fig:ablation}). Ours achieves peak performance at a scale of 2.0 for FID and CLIPScore, and at $\lambda = 1.5$ for PickScore. Additionally, increasing the value of ($\lambda$), the visual quality and text alignment are improved, as demonstrated in Figure \ref{fig:lambda}.
Based on these findings, we set the default configuration to ($\lambda$ = 2.0).

\noindent\textbf{$\beta$ and temperature} Besides the hyperparameters $\alpha$ and $\lambda$, we can alter $\beta$ (default = $1/\sqrt{d})$, which corresponds to $\ent$ with different temperatures (often referred to as inverse temperatures) \cite{sparsehop}. We find that our method is extendable to different $\beta$ (temperature). See supplement \ref{sec:temper}.


\section{Conclusion}
In this study, we introduce PLADIS, a novel approach to diffusion sampling that integrates the weight of sparse cross-attention, deviating from the dense cross-attention mechanism. Furthermore, by introducing a retrieval error bound in the case of 
$1 <\alpha \leq2$, we establish a connection between the noise robustness of sparse cross-attention in DMs. We provide in-depth analyses of sparsity in the cross-attention module for T2I generation. Building upon these analyses, we achieve significant improvements during inference time in generation across various guidance strategies and guidance-distilled models with our PLADIS. We believe PLADIS paves the way for future research in multimodal generation and alignment, with potential applications in domains requiring precise multimodal alignment via cross-attention.

\clearpage
\appendix
\onecolumn

\begin{center}
    \Large \textbf{\thetitle}\\
     {Supplementary Material}
\end{center}

\section{Supplementary Section} \label{sec:ablation}
In this supplementary document, we present the following:
\begin{itemize} 
\item Theoretical background on Hopfield energy networks and sparse Hopfield energy networks, the proof of the noise robustness in the intermediate cases, and the error bound of PLADIS in Section~\ref{sec:background}. 
\item Detailed description of the evaluation metrics and implementation in Section~\ref{sec:metric_detail}. 
\item Further detail and results of the user preference study in Section~\ref{sec:user_study}. 
\item Results for other backbone models including Stable Diffusion 1.5 and SANA, and combination with FreeU in Section~\ref{sec:add_other}. 
\item Results from one-step sampling with a guidance-distilled model in Section~\ref{sec:add_onestep}. 
\item Additional ablation studies, including attention temperature, cross-attention maps, the effect of layer selection, extrapolation strategy, and only sparse attention in Section~\ref{sec:add_ablation}. 
\item Additional qualitative results, including interactions with existing guidance sampling approaches, the guidance-distilled model, and further ablation studies in Section~\ref{sec:add_example}.

\end{itemize}

\section{Theoretical Background} \label{sec:background}

\paragraph{Notations.} For $a \in \Rd$, $a_{+}:=\max\{0,a\}$. For $\zb,\zb'\in \Rd^d$, $\langle \zb,\zb'\rangle=\zb^\intercal \zb'$ is the inner product of two vectors. For $\zb=(z_1,\dots,z_d)\in\Rd^d$, we denote the sorted coordinates of $\zb$ as $z_{(1)}\ge z_{(2)} \ge \dots \ge z_{(d)}$, that is, $z_{(\nu)}$ is the $\nu$'th largest element among $z_i$'s. $\Delta^M:=\{\pb \in \Rd^M| p_i\ge0, \sum p_i =1 \}$, $(M-1)$-dimensional simplex.

In this section, we provide the concept of modern Hopfield network and its sparse extension in simple form, to make readers fully understand the motivation and intuition of our method and encourage further research upon our works.

Initially, a Hopfiled model was introduced as an associative memory that can store binary patterns\cite{hopfield1982neural}.
The model is optimized to store patterns in the local minima of associated energy function.
Then, given query input, the closest local minimum point of the energy function is retrieved.
There were many extensions of the classic model to improve stability and capacity of the model, such as exponential energy functions or continuous state models\cite{demircigil2017model, krotov2016dense, barra2018new}.

Ramsauer et al. proposed modern Hopfield network that can be integrated into deep learning layers~\cite{hopfield}. The network is equipped with a new energy function $E$ and retrieval dynamics $\Tc$ that are differentiable and retrieve patterns after one update:
\begin{align}
E_{\texttt{Dense}}&: \Rd^d \to \Rd, \xb \mapsto -\texttt{lse}(\beta,\mbXi^{\top}\mbx) + \frac{1}{2} \langle \mbx,\mbx \rangle, \label{moder-energy}\\
\mt_{\texttt{Dense}}&: \Rd^d \to \Rd^d, \xb \mapsto \mbXi\texttt{Softmax}(\beta\mbXi^{\top}\mbx) \label{dense}
\end{align}
where $\xb \in \Rd^d$ represents a query input, $\Xib =[\xib_1 \dots \xib_M] \in \Rd^{d \times M}$, $\xib_i\in\Rd^d$ denotes a pattern stored, $\texttt{lse}(\beta,\mbz):=\log\left(\sum^M_{i=1}\exp(\beta z_i)\right) / \beta$ is log-sum-exponential function for $\beta > 0$ and $\texttt{Softmax}(\zb) := \frac{1}{\sum_{i=1}^d \exp(z_i)}(\exp(z_1),\dots,\exp(z_d))$, for $\zb \in \Rd^M$.
Theoretical results about the energy function and the retrieval dynamics including convergence, properties of states were proposed~\cite{hopfield}. 

\paragraph{{Connection with attention of the Transformer}}
Interesting connection between the update rule and self-attention mechanism used in transformer and BERT models was also proposed~\cite{hopfield}.
Specifically, we provide the detail derivation of this connection by following ~\cite{hopfield}. Firstly, we extend $\mt_{\texttt{Dense}}$ in Eq.~\ref{dense} to multiple queries $\Xb:= \{\xb_i\}_{i\in [N]}$.
Given any raw query $\Rb$ and memory matrix $\Yb$ that are input into Hopfield model, we calculate $\Xb$ and $\mbXi$ as $\Xb^{\top} = \Rb\Wb_Q:= \Qb, \mbXi^{\top}= \Yb\Wb_{K}:= \Kb$, using weight matrices, $\Wb_Q, \Wb_K$. Therefore, we rewrite $\mt_{\texttt{Dense}}$ as $ \Kb^{\top}\texttt{Softmax}(\beta\Kb\Qb^{\top})$.

Then, by taking transpose and projecting $\Kb$ to $\Vb$ with $\Wb_{V}$,  we have 
\begin{align}
   \mt_{\texttt{Dense}} : \Xb \mapsto  \texttt{Softmax}(\beta\Qb\Kb^{\top})\Kb\Wb_{V} = \texttt{Softmax}(\beta\Qb\Kb^{\top})\Vb, 
  \label{re2}
\end{align}
which is exactly transformer self-attention with $\beta = 1/\sqrt{d}$. In other words, we obtain by employing the notations in the~\cref{dense-atten},
\begin{align}
   \mt_{\texttt{Dense}} : \Xb \mapsto   \texttt{Softmax}(\Qb\Kb^{\top}/\sqrt{d})\Vb:= \Att(\Qb, \Kb, \Vb) 
   = \Att(\Wb_Q \Xb, \Wb_K \Xb, \Wb_V \Xb)
   \label{re3}
\end{align}

However, we can extend the interpretation to a cross-attention mechanism:
\[
\mt_{\texttt{Dense}} : (\Xb, \Yb) \mapsto 
\texttt{Softmax}\left(\Xb \Wb_Q  \Wb_K^{\top}\Yb^{\top} /\sqrt{d}\right)\Yb \Wb_V = \Att(\Wb_Q \Xb, \Wb_K \Yb, \Wb_V \Yb)
\]

We find similarity in the above cross-attention formula with inputs $\Xb, \Yb$ and weight matrices $\Wb_Q, \Wb_K, \Wb_V$. As discussed in lines of this paper, we focus on this extension into the cross-attention mechanism.

In terms of modern Hopefield network, the input query is processed with additional transformation $\Wb_Q$ to increase complexity of network and inner product are computed with stored (learned) $\Wb_K \Yb$ patterns (keys). Then, the retrieved patterns (values) for next layers are computed. Different layers can have different patterns, so hierarchical patterns are stored and retrieved in deep layers. Note that while Hopfield network outputs one pattern, the attention yields multiple patterns, so attention corresponds to stack of outputs of Hopfield network. Hence, the attention is multi-level and multi-valued Hopfield network.

\paragraph{Sparse Hopfield Network}
Later, sparse extensions of the modern Hopfield network are proposed~\cite{sparsehop, stanhop}.
The energy function was modified to make sparse the computation of retrieval dynamics:
\begin{align}
E_{\alpha} &: \Rd^d \to \Rd, \mbx \mapsto -\mathbf{\Psi}^{\star}_{\alpha}(\beta,\mbXi^{\top}\mbx) + \frac{1}{2} \langle \mbx,\mbx \rangle, \\
 \mt_{\alpha} &: \Rd^d \to \Rd^d, \mbx \mapsto \mbXi\alpha\texttt{-Entmax}(\beta\mbXi^{\top}\mbx), 
\end{align}
and $\mathbf{\Psi}^{\star}_{\alpha}$ is the convex conjugate of Tsallis entropy~\cite{tsallis}, $\mathbf{\Psi}_\alpha, \alpha\texttt{-Entmax}(\mathbf{z})$,  represents the probability mapping:
\begin{align}
&\Psi_{\alpha}(\mathbf{p}):=
\begin{cases}
\frac{1}{\alpha(\alpha-1)}\sum^{M}_{i=1} (p_i - p^{\alpha}_i), \; &\alpha \neq 1, \\
-\sum^{M}_{i=1} (p_i - \log p_i), &\alpha = 1, 
\end{cases}
\end{align}
\begin{align}
\alpha\texttt{-Entmax}(\mbz) := \underset {\mbp\in \Delta^{M}}{\argmax} [\langle\mbp,\mbz\rangle-\Psi_{\alpha}(\mbp) ],
\end{align}
where $\mathbf{p} \in \mathbb{R}^{M}$.
Here, $\alpha$ controls the sparsity.
When $\alpha=1$, it is equivalent to a dense probability mapping, $1\texttt{-Entmax} = \texttt{Softmax}$, and as $\alpha$ increases towards $2$, the outputs of $\alpha\texttt{-Entmax}$ become increasingly sparse, ultimately converging to $2\texttt{-Entmax} \equiv \texttt{Sparsemax}(\zb):= \underset{\pb \in \Delta^{M}}{\argmin} \norm{\pb - \zb }$~\cite{sparsemax}.
Notably, when $\alpha = 1$, $\mt_{\alpha}$ becomes equivalent to $\mt_{\texttt{Dense}}\equiv \Tc_1$~\cite{varsoftmax}.
We have simple formula for $\alpha\texttt{-Entmax}$\cite{sparsemax}.
There is a unique threshold function $\tau: \Rd^M \to \Rd$ that satisfies
\begin{align}
    \alpha\texttt{-Entmax}(\zb) = [(\alpha-1)\zb - \tau(\zb) \boldsymbol{1}]_+^{1/(\alpha-1)}. \label{eq:entmax}
\end{align}
From this formula, we know that the entries less than $\tau/(\alpha-1)$ map to zero, so sparsity is achieved.
We will denote the number of nonzero entries in $\alpha\texttt{-Entmax}$ as $\kappa(\zb)$ for later use to derive theoretical results.
For $\alpha = 2$, the exact solution can be efficiently computed using a sorting algorithm~\cite{held1974validation,michelot1986finite}.
For $1 < \alpha < 2$, inaccurate and slow iterative algorithm was used for computing $\alpha\texttt{-Entmax}$~\cite{liu2009efficient}.
Interestingly, for $1.5\texttt{-Entmax}$, an accurate and exact solution are derived in a simple form~\cite{entmaxx}.

Similar to $\mt_{\texttt{Dense}}$, $\Tca$ can be extended to attention mechanisms, establishing a strong connection with sparse attention.
In other words, by following the derivation as provided in ~\cref{re2}, and \cref{re3}, we can obtain
\begin{align}
       \Tca : \Xb \mapsto   \ent(\Qb\Kb^{\top}/\sqrt{d})\Vb:= \Att_{\alpha}(\Qb, \Kb, \Vb)
\end{align}
Furthermore, similar to the dense attention mechanism, we can also extend into a cross-attention mechanism with inputs $\Xb$ and $\Yb$:
\[
\Tca : (\Xb, \Yb) \mapsto 
\ent\left(\Xb \Wb_Q  \Wb_K^{\top}\Yb^{\top} /\sqrt{d}\right)\Yb \Wb_V = \Att_{\alpha}(\Wb_Q \Xb, \Wb_K \Yb, \Wb_V \Yb)
\]

\paragraph{Noise robustness of sparse Hopfield network}
In SHN, sparsity reduces retrieval errors and provide faster convergeness compared to dense retrieval dynamics~\cite{sparsehop,stanhop}.
While the sparse extension is an efficient counterpart of dense Hopfield network, it has been discovered that there is more advantages to use sparse one besides efficiency~\cite{sparsehop,stanhop}. 

\begin{definition}[Pattern Stored and Retrieved]
    Suppose every pattern $\xib_\mu$ is contained in a ball $B_\mu$. We say that $\xib_\mu$ is stored if there is a single fixed point $\xb_i^* \in B_\mu$, to which all point $\xb\in B_\mu$ converge, and $B_\mu$'s are disjoint. We say that $\xib_\mu$ is retrieved for an error $\epsilon$ if $||\Tc(\xb)-\xib_\mu||\le\epsilon$ for all $\xb\in B_\mu$
\end{definition}
For following theorems, $m:= \max_\nu ||\xib_\nu||.$
\begin{thm}[Retrieval Error]{\cite{hopfield,sparsehop,stanhop}}
    Let $\Tca$ be the retrieval dynamics of Hopfield model with $\ent$.
    \begin{align}
        &\mbox{For } \alpha = 1,
        &||\Tca(\xb)-\xib_\mu|| &\le
        2m (M-1) \exp \left\{ -\beta \left( \langle\xib_\mu, \xb \rangle - \max_\nu \langle \xib_\mu, \xib_\nu \rangle \right) \right\} . \label{eq:ineq-hop}\\
        &\text{For } \alpha = 2,
        &||\Tca(\xb)-\xib_\mu|| &\le m + m \beta \left[ \kappa \left( \max_\nu \langle \xib_\nu, \xb \rangle - \left[\Xib^\intercal \xb \right]_{(\kappa)} \right) + \frac{1}{\beta}\right] . \label{eq:ineq-sparse}\\
        &\text{For } \alpha > \alpha',
        &|| \Tca(x) - \xib_\mu || &\le ||\Tc_{\alpha'} - \xi || . \label{eq:ineq-mono}
    \end{align}
\end{thm}
You can find the result \cref{eq:ineq-hop} in \cite{hopfield}, \cref{eq:ineq-sparse} in \cite{sparsehop}, and \cref{eq:ineq-mono} in \cite{sparsehop, stanhop}.

\begin{corollary}(Noise-Robustness)~\cite{sparsehop, stanhop}.
In case of noisy patterns with noise $\boldsymbol{\eta}$, i.e. $\tilde{\mbx} = \mbx + \boldsymbol{\eta}$ (noise in query) or $\tilde{\mbxi}_{\mu} = \mbxi_{\mu} + \boldsymbol{\eta}$ (noise in memory), the impact of noise $\boldsymbol{\eta}$ on the sparse retrieval error $||\Tc_2(\xb) - \xib_\mu|$ is linear, while its effect on the dense retrieval error $||\Tc_1(\xb) - \xib_\mu||$ is exponential.
\end{corollary}
\noindent where $\mbxi_{\mu}$ is memory pattern and to be considered stored at a fixed point of $\mt$. This theorem suggests that under noisy conditions, sparse attention mechanisms governed by $\Tca$ with $\alpha>1$ exhibit superior noise robustness compared to standard dense attention.  Critically, increasing sparsity (via higher $\alpha$) further diminishes retrieval errors.

We propose a new theoretical result that completes above theorem by providing error estimation for all intermediate cases that was not given.
\begin{thm}[Retrieval Error 2] \label{thm:error-new}
    Let $\Tca$ be the retrieval dynamics of Hopfield model with $\ent$. 
    \begin{align}
        &\text{For } 1<\alpha \le 2,
        &||\Tca(\xb)-\xib_\mu|| &\le m +
        m \kappa \left[ (\alpha-1)\beta  \left( \max_\nu \langle \xib_\nu, \xb \rangle - \left[\Xib^\intercal \xb \right]_{(\kappa+1)} \right)\right]^{\frac{1}{\alpha-1}},\label{eq:ineq-new1}
    \end{align}
    Here, we abuse the notation $\left[\Xib^\intercal \xb \right]_{(M+1)}:= \left[\Xib^\intercal \xb \right]_{(M)} - M^{1-\alpha}/(\alpha-1)$.
\end{thm}

Thanks to this new theorem, we can estimate the impact of noise on the sparse retrieval error for all $1<\alpha<2$.

\begin{corollary}(Noise-Robustness)
In case of noisy patterns with noise $\boldsymbol{\eta}$, the impact of noise $\boldsymbol{\eta}$ on the retrieval error $||\Tca(\xb) - \xib_\mu||$ is polynomial of order $\frac{1}{\alpha-1}$ for $1<\alpha \le 2$.
\end{corollary}
\paragraph{{Remark}} The proposed theorem includes the case $\alpha=2$. In that case, the right hand side becomes $$m \beta \left[\kappa \left( \max_\nu \langle \xib_\nu, \xb\rangle - [\Xib^\intercal \xb]_{(\kappa+1)} \right) \right].$$
Therefore, by combining with previous result, we obtain tighter bound:
\[
||\Tc_2(\xb)-\xib_\nu|| \le m \beta \left[\kappa  \max_\nu \langle \xib_\nu, \xb\rangle + \min \left\{ -\kappa [\Xib^\intercal \xb]_{(\kappa+1)}, -\kappa [\Xib^\intercal \xb]_{(\kappa)} + \frac{1}{\beta} \right\}  \right]
\]

\begin{proof}[proof of Thm. \ref{thm:error-new}]
    \begin{align}
        ||\Tca(\xb)-\xib_\mu|| &= \norm{ \Xib \alpha\texttt{-Entmax}\left(\beta \Xib^\intercal \xb\right) -\xib_\mu }
        = \norm{ \sum_{\nu=1}^\kappa \xib_{(\nu)}\left[\alpha\texttt{-Entmax}\left( \beta\Xib^\intercal \xb\right)\right]_{(\nu)} - \xib_\mu }\\
        &\le ||\xib_\mu|| + \sum_{\nu=1}^\kappa \norm{\xib_{(\nu)}}\left[\alpha\texttt{-Entmax}\left( \beta\Xib^\intercal \xb\right)\right]_{(\nu)}\\
        &\le m + m \sum_{\nu=1}^\kappa \left[ (\alpha-1)\left(\left[ \beta\Xib^\intercal \xb\right]_{(\nu)} - \left[ 
        \beta\Xib^\intercal \xb\right]_{(\kappa+1)} \right)\right]^{\frac{1}{\alpha-1}} \label{line-alpha}\\
        &\le m + m \kappa \max_\nu \left[ (\alpha-1)\beta \left(\langle \xib_\nu, \xb \rangle - \left[ \Xib^\intercal \xb\right]_{(\kappa+1)} \right)\right]^{\frac{1}{\alpha-1}}.
    \end{align}
    For \cref{line-alpha}, we use the following lemma.
\end{proof}
\begin{lem}\label{lem-alpha}
    For $\zb \in \Rd^M$ and $\nu \le \kappa(\zb)$, $[\alpha\texttt{-Entmax}(\zb)]_{(\nu)} \le [(\alpha-1)(z_{(\nu)} - z_{\left(\kappa+1\right)})]^{1/(\alpha-1)}$.
\end{lem}
\begin{proof}\mbox{}\\
\begin{enumerate}
\item[(i)]$\kappa<M$\\
From the definition of $\kappa$, we have following properties.
$$\alpha\texttt{-Entmax}(\zb)_{(\kappa+1)}=0.$$
$$z_{(\kappa+1)} \le \tau(\zb)/(\alpha-1).$$
Keep the last inequality, and now consider the $\nu$'th largest coordinate of \cref{eq:entmax}, but we can omit $+$ since it is strictly positive.
\begin{align*}
    \alpha\texttt{-Entmax}(\zb)_{(\nu)} &= [(\alpha-1)z_{(\nu)} - \tau(\zb)]_+^{1/(\alpha-1)}\\
    &=[(\alpha-1)z_{(\nu)} - \tau(\zb)]^{1/(\alpha-1)}\\
    &\le [(\alpha-1)z_{(\nu)} - (\alpha-1)z_{(\kappa+1)}]^{1/(\alpha-1)}
\end{align*}
\item[(ii)] $\kappa = M$\\
We use H\"older inequality $$ \left(\sum|a_i|^p\right)^{1/p} \left(\sum|b_i|^q\right)^{1/q} \ge \sum |a_i b_i| \quad \text{ for } p,q\in (1,\infty), 1/p+1/q=1 $$
to estimate a lower bound of $\tau$ for $\alpha\neq2$.
By substituting $a_i = (\alpha-1)z_i -\tau, b_i = 1, p=1/(\alpha-1), q=1/(2-\alpha)$,
\[
\left(\sum|(\alpha-1)z_i -\tau|^{1/(\alpha-1)}\right)^{\alpha-1} \left(\sum1\right)^{2-\alpha} \ge \sum |(\alpha-1)z_i -\tau|.
\]
We know that all entries are positive $(\alpha-1)z_i -\tau > 0$ since $\kappa = M$. Moreover, 
\[
\sum[(\alpha-1)z_i -\tau]^{1/(\alpha-1)}=1
\]
since the left hand side is the sum of the coordinates of $\alpha\texttt{-Entmax}$ output.
Therefore,
\begin{align*}
M^{2-\alpha} &\ge (\alpha-1)\sum z_i - M\tau\\
\frac{\tau}{\alpha-1} &\ge \frac1M \sum z_i  - \frac{M^{1-\alpha}}{\alpha-1}\\
&\ge \min z_i - \frac{M^{1-\alpha}}{\alpha-1} = z_{(M)}- \frac{M^{1-\alpha}}{\alpha-1}
\end{align*}
We remain the case $\alpha=2$. We directly sum up the entries of $2\texttt{-Entmax}$:
\begin{align*}
    1 = \sum |z_i - \tau| &= \sum z_i - M\tau \\
    &\ge M \min z_i - M \tau\\
    \therefore \tau &\ge z_{(M)} - \frac{1}{M} = z_{(M)}- \frac{M^{1-\alpha}}{\alpha-1}
\end{align*}
\end{enumerate}
\end{proof}

We further estimate the retrieval error of retrieval dynamics defined in PLADIS. We use the notation:

\[
\Tcal(\xb):= \lambda \Tca(\xb) + (1-\lambda)\Tc_1 (\xb).
\]
Then, we have following result for the retrieval error of $\Tc^\lambda_\alpha$.
\begin{thm}[Retrieval Error 3]
    Consider the retrieval dynamics $\Tcal$
    \begin{align}
        ||\Tcal(\xb)-\xib_\mu|| &\le |\lambda| m +
        |\lambda| m \kappa \left[ (\alpha-1)\beta  \left( \max_\nu \langle \xib_\nu, \xb \rangle - \left[\Xib^\intercal \xb \right]_{(\kappa+1)} \right)\right]^{\frac{1}{\alpha-1}}\\ &+ |1-\lambda| 2m (M-1) \exp \left\{ -\beta \left( \langle\xib_\mu, \xb - \max_\nu \langle \xib_\mu, \xib_\nu \rangle \right) \right\} .\label{eq:ineq-new2}
    \end{align}
\end{thm}
\begin{proof}
    \begin{align*}
        ||\Tcal(\xb)-\xib_\nu||
        &= ||\lambda \Tca(\xb) + (1-\lambda) \Tc_1(\xb) - \xib_\nu|| \\
        &\le |\lambda| || \Tca(\xb) + \xib_\nu || + |1-\lambda| ||\Tc_1(\xb) - \xib_\nu||
    \end{align*}
    and apply \cref{eq:ineq-hop} and \cref{eq:ineq-new1}.
\end{proof}
This theorem suggests that the retrieval dynamics given in PLADIS have the error bound of mixture of polynomial and exponential terms.

\section{Metrics and Implementation Detail} \label{sec:metric_detail}
For image sampling in Table \ref{tab_main}, sampling without CFG guidance is conducted using 30,000 randomly selected text prompts from the MSCOCO validation dataset. Conversely, sampling with CFG is performed with uniformly selected values of $w$ in the range (3,5). In both cases, the PAG and SEG scales are fixed at 3.0, following the recommended settings from the corresponding paper.

For Tables \ref{tab_hps} and \ref{tab_distill}, we use 200 prompts from Drawbench \cite{drawbench}, 400 prompts from HPD \cite{hpsv2}, and 500 prompts from the test set of Pick-a-pic \cite{pick}, generating 5 images per prompt. Additionally, for the ablation study in Table \ref{tab_alpha}, we generate 5,000 images from the MSCOCO validation set with CFG and PAG guidance. As with Table \ref{tab_main}, the CFG scale is uniformly selected within the range of (3,5), while the PAG scale remains set at 3.0.

\section{User Preference Study} \label{sec:user_study}
As presented in Fig.~\ref{fig:user_study}, we employ human evaluation and do not rely solely on automated evaluation metrics such as FID, CLIPScore, ImageReward, etc. Our aim is to assess whether PLADIS truly improves image quality and prompt coherence. To rigorously evaluate these aspects, we categorized caess into two groups: interaction with guidance sampling including CFG~\cite{CFG}, PAG~\cite{PAG}, SEG~\cite{SEG}, and interaction with guidance-distilled models such as SDXL-Turbo~\cite{turbo}, SDXL-Lightening~\cite{light}, DMD2~\cite{dmd2}, and Hyper-SDXL~\cite{hyper}. We evaluate all models based on 20 selected prompts from the randomly selected Drawbench~\cite{drawbench}, HPD~\cite{hpsv2}, and Pick-a-pic~\cite{pick}. For the guidance-distilled model, we select half from one-step sampling results and the other half from four-step sampling results. Human evaluators, who are definitely blind and anonymous, are restricted to participating only once. Evaluators are shown two images from model outputs with and without PLADIS based on the same text prompt and measure images with two questions: for image quality, "Which image is of higher quality and visually more pleasing?" and for prompt alignment, "Which image looks more representative of the given prompt." The order of prompts and the order between models are truly randomized. In Fig.~\ref{fig:user_study}, we averaged all of the results related to the guidance-distilled model due to limited space. Further presenting in detail, we present a user preference study for each guidance-distilled model as shown in Fig.~\ref{fig:user_study_supple}. As similar to guidance sampling, guidance-distilled models with PLADIS outperform both image quality and prompt alignment, validating the practical effectiveness of PLADIS.

\begin{figure}[ht!]
\centering
\includegraphics[width=1\linewidth]{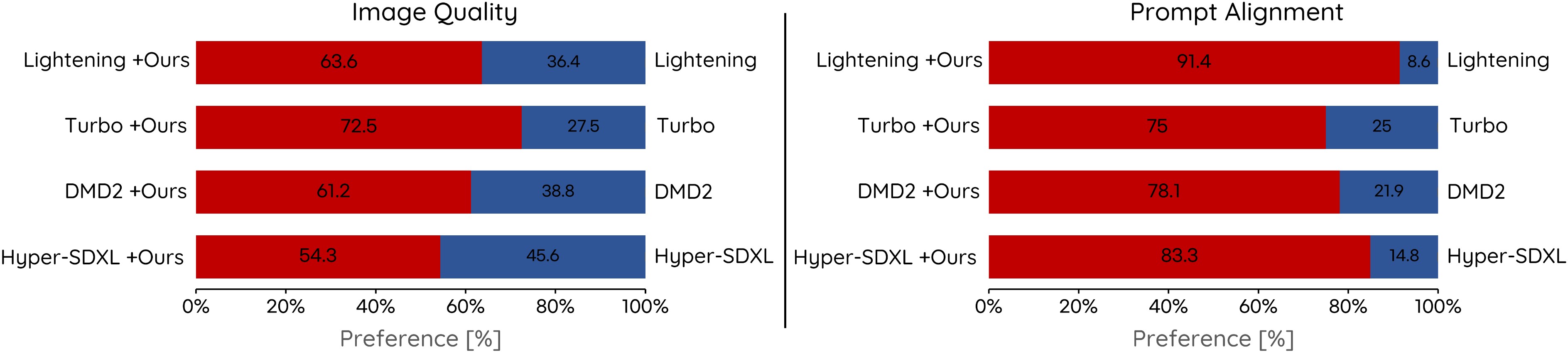}
\caption{User preference study for PLADIS in the context of guidance-distilled models.
We evaluate the two aspects of model output with and without PLADIS such as image quality and prompt alignment. }\label{fig:user_study_supple}
\end{figure}


\section{Application on Other Backbone} \label{sec:add_other}

To demonstrate the robustness of our proposed method, we perform experiments using additional backbones, including Stable Diffusion v1.5 (SD1.5) and SANA~\cite{sana}. SANA is a recently introduced text-to-image diffusion model that uses linear attention, enabling faster image generation. It is based on the Diffusion Transformer (DiT) architecture. We generate 30K samples from randomly selected MS COCO validation set images and evaluate them using FID, CLIPScore, and ImageReward, as shown in Table~\ref{tab_another}. For SD1.5, we use CFG, while SANA is tested with its default configuration without modifications.

Interestingly, we observe that both SD1.5 and SANA, when integrated with our PLADIS method, consistently improve performance across all metrics. A visual comparison is provided in Fig. \ref{fig:supple_sd15} and Fig. \ref{fig:supple_sana}. As shown in the figures, the generation with our PLADIS provides more natural and pleasing images and precise matching between images and text prompts on both backbones.
As seen in other experiments, our PLADIS enhances both generation quality and text alignment with the given prompts. By confirming these improvements with SD1.5 and SANA, we demonstrate that PLADIS is robust across different backbones, particularly transformer-based architectures.

\begin{table*}[ht]
\caption{Quantitative comparison across various datasets using 1-steps sampling with the guidance-distilled model.}
\vspace{-0.5em}
\label{tab_add_distill}
\centering
\resizebox{1\linewidth}{!}{
\begin{small}
\begin{tabular}{lccccccccc}
\toprule
 & \multicolumn{3}{c}{Drawbench~\cite{drawbench}} & \multicolumn{3}{c}{HPD~\cite{hpsv2}} & \multicolumn{3}{c}{Pick-a-pic~\cite{pick}} \\
\cmidrule(r){2-4}\cmidrule(r){5-7}\cmidrule(r){8-10}
Method &  CLIPScore~$\uparrow$ & PickScore~$\uparrow$ & ImageReward~$\uparrow$  &  CLIPScore~$\uparrow$ & PickScore~$\uparrow$ & ImageReward~$\uparrow$  &  CLIPScore~$\uparrow$ & PickScore~$\uparrow$ & ImageReward~$\uparrow$ \\
\cmidrule(r){1-1}\cmidrule(r){2-4}\cmidrule(r){5-7}\cmidrule(r){8-10}
Turbo~\cite{turbo} & 27.19 & 21.67 & 0.305	& 28.45 &	21.85	& 0.479	& 26.89	& 21.16	& 0.346 \\
\cmidrule(r){1-1}\cmidrule(r){2-4}\cmidrule(r){5-7}\cmidrule(r){8-10}
\rowcolor{green!10}+\,Ours & \bf{27.56} \scriptsize{\textcolor{blue}{(+0.37)}}& \bf{21.68} \scriptsize{\textcolor{blue}{(+0.01)}} &	\bf{0.390} \scriptsize{\textcolor{blue}{(+0.08)}}&	\bf{28.78} \scriptsize{\textcolor{blue}{(+0.33)}}& \bf{21.86} \scriptsize{\textcolor{blue}{(+0.01)}}& \bf{0.517} \scriptsize{\textcolor{blue}{(+0.04)}}&\bf{27.10} \scriptsize{\textcolor{blue}{(+0.21)}}& \bf{21.17} \scriptsize{\textcolor{blue}{(+0.01)}}&	\bf{0.378} \scriptsize{\textcolor{blue}{(+0.04)}}
 \\
\midrule
Light~\cite{light}&26.08	&21.86	&0.428	&27.37	&22.05	&0.730	&25.73	&21.34	&0.585 \\
\cmidrule(r){1-1}\cmidrule(r){2-4}\cmidrule(r){5-7}\cmidrule(r){8-10}
\rowcolor{green!10}+\,Ours & \bf 26.66 \scriptsize{\textcolor{blue}{(+0.58)}}&	\bf 21.94 \scriptsize{\textcolor{blue}{(+0.08)}}&	\bf0.558 \scriptsize{\textcolor{blue}{(+0.13)}} &\bf	28.42 \scriptsize{\textcolor{blue}{(+1.05)}}& \bf	22.24 \scriptsize{\textcolor{blue}{(+0.19)}} & \bf 0.830 \scriptsize{\textcolor{blue}{(+0.10)}} & \bf 26.63 \scriptsize{\textcolor{blue}{(+0.90)}}&	\bf 21.46	\scriptsize{\textcolor{blue}{(+0.12)}}& \bf 0.680 \scriptsize{\textcolor{blue}{(+0.10)}}
 \\
\midrule
DMD2~\cite{dmd2}& 27.91	&22.04	&0.651	&29.95	&22.18	&0.888	&28.14	&21.57	&0.770 \\
\cmidrule(r){1-1}\cmidrule(r){2-4}\cmidrule(r){5-7}\cmidrule(r){8-10}
\rowcolor{green!10}+\,Ours & \bf28.09 
\scriptsize{\textcolor{blue}{(+0.19)}}& \bf	22.05 \scriptsize{\textcolor{blue}{(+0.01)}}& \bf 0.662 
\scriptsize{\textcolor{blue}{(+0.01)}}& \bf	30.21 	\scriptsize{\textcolor{blue}{(+0.26)}}& \bf 22.20 \scriptsize{\textcolor{blue}{(+0.02)}}& \bf	0.902 \scriptsize{\textcolor{blue}{(+0.01)}}&	\bf 28.38  \scriptsize{\textcolor{blue}{(+0.43)}}&	\bf 21.58 \scriptsize{\textcolor{blue}{(+0.01)}}& \bf	0.794 \scriptsize{\textcolor{blue}{(+0.02)}}
 \\
\midrule
Hyper~\cite{hyper} &27.41	&22.27	&0.662	&29.09	&22.61	&0.912	&27.29	&21.91	&0.812 \\
\cmidrule(r){1-1}\cmidrule(r){2-4}\cmidrule(r){5-7}\cmidrule(r){8-10}
\rowcolor{green!10}+\,Ours & \bf 27.80 
\scriptsize{\textcolor{blue}{(+0.39)}}& \bf	22.30 \scriptsize{\textcolor{blue}{(+0.03)}}&	\bf 0.674 \scriptsize{\textcolor{blue}{(+0.01)}}& \bf	29.42 \scriptsize{\textcolor{blue}{(+0.33)}}& \bf 22.65 \scriptsize{\textcolor{blue}{(+0.04)}}&	\bf 0.932 \scriptsize{\textcolor{blue}{(+0.02)}}&	\bf 27.85 \scriptsize{\textcolor{blue}{(+0.56)}}&	\bf 21.92 \scriptsize{\textcolor{blue}{(+0.01)}}& \bf	0.832 \scriptsize{\textcolor{blue}{(+0.02)}}
 \\
\bottomrule	        
\end{tabular}
\end{small}
}
\end{table*}

\section{Comparison Results on One-Step Sampling} \label{sec:add_onestep}
As discussed in Section \ref{sec:distill}, we found that our proposed method, PLADIS, is also effective for one-step sampling with a guidance-distilled model. Following the experimental settings in Table \ref{tab_distill}, we generate images from text prompts in human preference datasets such as Drawbench \cite{drawbench}, HPD \cite{hpsv2}, and Pick-a-pick \cite{pick}. The generated images are evaluated using CLIPScore, ImageReward, and PickScore, as presented in Table \ref{tab_add_distill}. Our method consistently yields performance improvements, particularly in text alignment and human preference, across all baselines. This demonstrates the robustness of our approach for denoising steps and highlights its potential as a generalizable boosting solution.

\begin{table}[t!]
\begin{minipage}[t]{.48\linewidth}
\caption{Application on other BackBone Model on MS COCO validation set and Comparison results for another extrapolation strategy and combination with FreeU~\cite{freeu}. SD1.5 and SANA indicate that Stable Diffusion version 1.5 and SANA 1.6 B model, respectively.}
\vspace{-0.5em}
\label{tab_another}
\centering
\resizebox{1\linewidth}{!}{
\begin{small}
\begin{tabular}{ccccc}
\toprule
Resolution & BackBone & FID $\downarrow$ & CLIPScore  $\uparrow$  & ImageReward $\uparrow$  \\
\cmidrule(r){1-1}\cmidrule(r){2-5}
\multirow{2}{*}{512 $\times$ 512} & SD1.5 & 23.88 & 24.11 & -0.368
\\
\cmidrule(r){2-2}\cmidrule(r){3-5}
\rowcolor{green!10} \cellcolor{white} & + PLADIS (Ours)  & 22.41\scriptsize{\textcolor{blue}{(-1.48)}} & 25.09 \scriptsize{\textcolor{blue}{(+0.98)}}& -0.08 \scriptsize{\textcolor{blue}{(+0.360)}} \\
\cmidrule(r){1-1}\cmidrule(r){2-2}\cmidrule(r){3-5}
\multirow{2}{*}{1024 $\times$ 1024}& SANA~\cite{sana} & 28.01 & 26.61 & 0.867 \\
\cmidrule(r){2-2}\cmidrule(r){3-5}
 \rowcolor{green!10} \cellcolor{white}& + PLADIS (Ours)  & 27.53\scriptsize{\textcolor{blue}{(-0.48)}} & 26.83 \scriptsize{\textcolor{blue}{(+0.21)}}& 0.883\scriptsize{\textcolor{blue}{(+0.016)}} \\
\midrule
Resolution & Method & FID $\downarrow$& CLIPScore $\uparrow$  & ImageReward $\uparrow$\\
\cmidrule(r){1-1}\cmidrule(r){2-2}\cmidrule(r){3-5}
\multirow{3}{*}{1024 $\times$ 1024} & SDXL (CFG)     & 32.68 & 25.90 & 0.425\\
\cmidrule(r){2-2}\cmidrule(r){3-5}
& + Ours (Prediction)      & 29.48 & 26.60 & 0.619 \\
\rowcolor{green!10} \cellcolor{white} &  + Ours (In-model)     & \bf28.50 & \bf26.61 & \bf0.626 \\
\midrule
\multirow{2}{*}{1024 $\times$ 1024} & SDXL + FreeU     & 35.66 & 25.96 & 0.425\\
\cmidrule(r){2-2}\cmidrule(r){3-5}
\rowcolor{green!10} \cellcolor{white} & + PLADIS (Ours)     & \bf 28.79 & \bf 26.93 & \bf0.626 \\
 \bottomrule	        
\end{tabular}
\end{small}
}
\end{minipage}
\hfill
\begin{minipage}[t]{.48\linewidth}
\caption{Ablation study on layer group which is replaced with PLADIS on MS COCO validation dataset.}
\vspace{-0.5em}
\label{tab_layer}
\centering
\resizebox{1\linewidth}{!}{
\begin{small}
\begin{tabular}{cccc}
\toprule
Layer & FID $\downarrow$ & CLIPScore $\uparrow$  & ImageReward $\uparrow$  \\
\cmidrule(r){1-1}\cmidrule(r){2-4}
Baseline & 33.76 & 25.41 & 0.478 \\
\cmidrule(r){1-1}\cmidrule(r){2-4}
Up  & 29.78\scriptsize{\textcolor{blue}{(-3.98)}} & 25.78 \scriptsize{\textcolor{blue}{(+0.37)}}& 0.624\scriptsize{\textcolor{blue}{(+0.15)}} \\
Mid  & 31.76\scriptsize{\textcolor{blue}{(-2.00)}} & 25.46 \scriptsize{\textcolor{blue}{(+0.05)}}& 0.496\scriptsize{\textcolor{blue}{(+0.02)}} \\
Down  & 31.46\scriptsize{\textcolor{blue}{(-2.30)}} & 25.43 \scriptsize{\textcolor{blue}{(+0.02)}}& 0.501\scriptsize{\textcolor{blue}{(+0.02)}} \\
Up, Mid & 30.76\scriptsize{\textcolor{blue}{(-3.00)}} & 25.46 \scriptsize{\textcolor{blue}{(+0.05)}}& 0.548\scriptsize{\textcolor{blue}{(+0.07)}} \\
Up, Down  & 28.46\scriptsize{\textcolor{blue}{(-5.30)}} & 26.12 \scriptsize{\textcolor{blue}{(+0.71)}}& 0.658\scriptsize{\textcolor{blue}{(+0.18)}} \\
Mid, Down  & 31.36\scriptsize{\textcolor{blue}{(-2.40)}} & 25.52 \scriptsize{\textcolor{blue}{(+0.11)}}& 0.498\scriptsize{\textcolor{blue}{(+0.02)}} \\
\cmidrule(r){1-1}\cmidrule(r){2-4}
\rowcolor{green!10} All (Ours)  & 27.87\scriptsize{\textcolor{blue}{(-5.89)}} & 26.41 \scriptsize{\textcolor{blue}{(+1.00)}}& 0.726\scriptsize{\textcolor{blue}{(+0.25)}} \\
\bottomrule	        
\end{tabular}
\end{small}
}
\end{minipage}
\end{table}

\section{Additional Ablation Study} \label{sec:add_ablation}

\begin{figure}[ht!]
\centering
\includegraphics[width=0.85\linewidth]{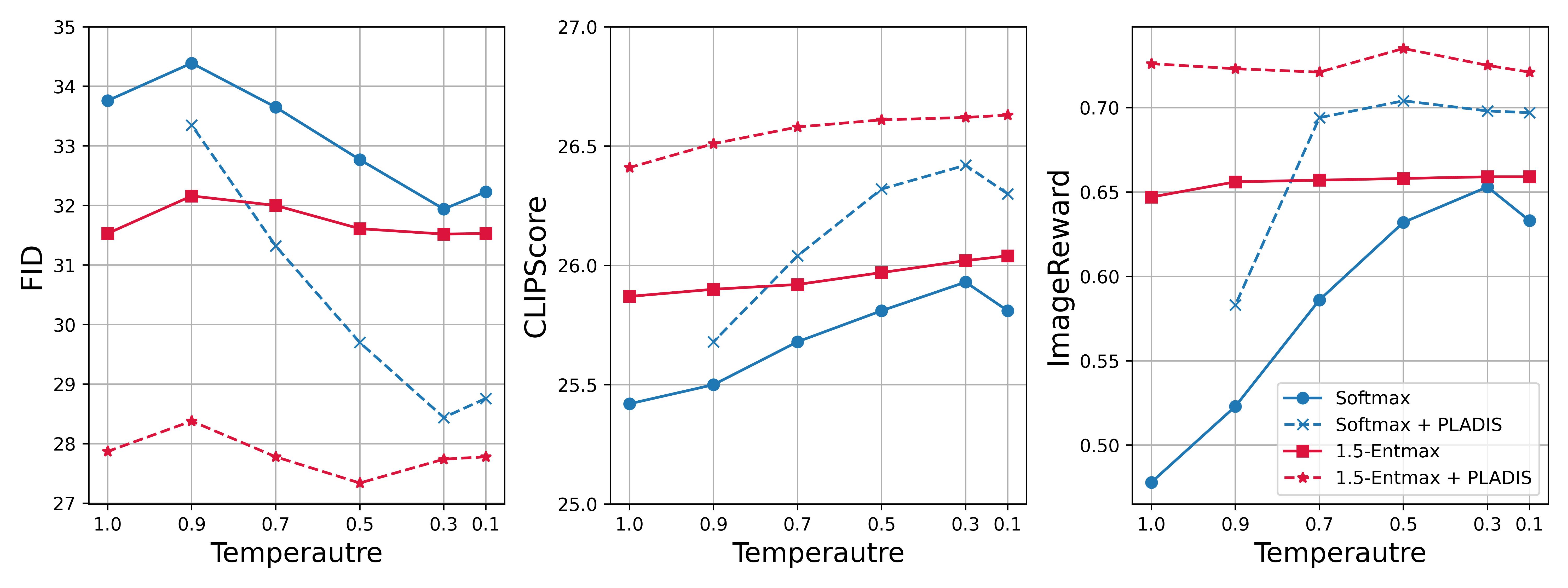}
\caption{Comparison results for various temperatures, with and without PLADIS, are presented, including the baseline (\texttt{Softmax}) and 1.5$-\texttt{Entmax}$. While lower temperatures with the baseline offer benefits in both cases, our proposed method ($\alpha$ = 1.5), with and without PLADIS, outperforms across all temperature settings.}\label{fig:temper_supple}
\end{figure}

\subsection{Comparison with Attention Temperature}\label{sec:temper}
In the field of NLP, to improve existing attention mechanisms, temperature scaling~\cite{temper}, also known as inverse temperature, has been extensively studied to adjust the sharpness of attention. It is defined as follows:
\begin{align}
\Att(\Qb,\Kb,\Vb) = \texttt{Softmax}(\frac{\Qb\Kb^{\top}}{\sqrt{d}*\tau})
\end{align}
where $\tau$ denotes the temperature, which controls the softness of the attention. A lower temperature results in sharper activations, creating a more distinct separation between values. Importantly, it is closely related to the $\beta$ in $\ent$. In common attention mechanisms, $\beta$ is typically set to the square root of the dimension, $\sqrt{d}$, which corresponds to $\tau = 1.0$. In modern sparse Hopfield energy functions, $\beta$ serves as a scaling factor for the energy function, influencing the sharpness of the energy landscape and thereby controlling the dynamics~\cite{sparsehop}. \citeauthor{sparsehop} argue that high $\beta$ values, corresponding to low temperatures ($\tau < 1$), help maintain distinct basins of attraction for individual memory patterns, facilitating easier retrieval.

As discussed in the main paper, we provide an ablation study on the hyperparameter $\tau$ (which is equivalent to $\beta$) by varying $\tau$ from 0.9 to 0.1 for $\texttt{Softmax}$, alongside our default configuration (1.5$-\texttt{Entmax}$). Similar to the previous ablation study, we generate 5K images from randomly selected samples in the MS-COCO validation set under CFG and PAG guidance with our PLADIS, as shown in Fig.~\ref{fig:temper_supple}.

We observed that lowering the temperature (increasing $\beta$) consistently improved generation performance in both transformations, such as \texttt{Softmax} and 1.5$-\texttt{Entmax}$. In the case without PLADIS, \texttt{Softmax} with a lower temperature improved all metrics, but its performance still remained inferior to sparse attention ($\alpha$ = 1.5).
When using PLADIS, the trend was similar: \texttt{Softmax} with a lower temperature benefited from PLADIS, but it still did not outperform the 1.5$-\texttt{Entmax}$ configuration with PLADIS.

Furthermore, 1.5$-\texttt{Entmax}$ with a lowered temperature consistently improves generation quality in terms of visual quality and text alignment, ultimately converging to similar performance. Notably, very low temperatures with \texttt{Softmax} result in nearly identical sparse transformations, but with larger-than-zero intensities. This suggests that lowering the temperature benefits all transformations in $\ent$ for $1 \leq \alpha \leq 2$. However, dense alignment with a lowered temperature is insufficient, and sparse attention remains necessary in both cases, with and without PLADIS. Additionally, adjusting other hyperparameters is time-consuming, but our PLADIS with 1.5$-\texttt{Entmax}$ does not require finding the optimal hyperparameter $\tau$, thanks to the convergence of performance across various $\tau$ values. Therefore, these results demonstrate that the noise robustness of sparse cross-attention in diffusion models (DMs) is crucial for generation performance.

\begin{figure*}[ht!]
\centering
\includegraphics[width=0.95\linewidth]{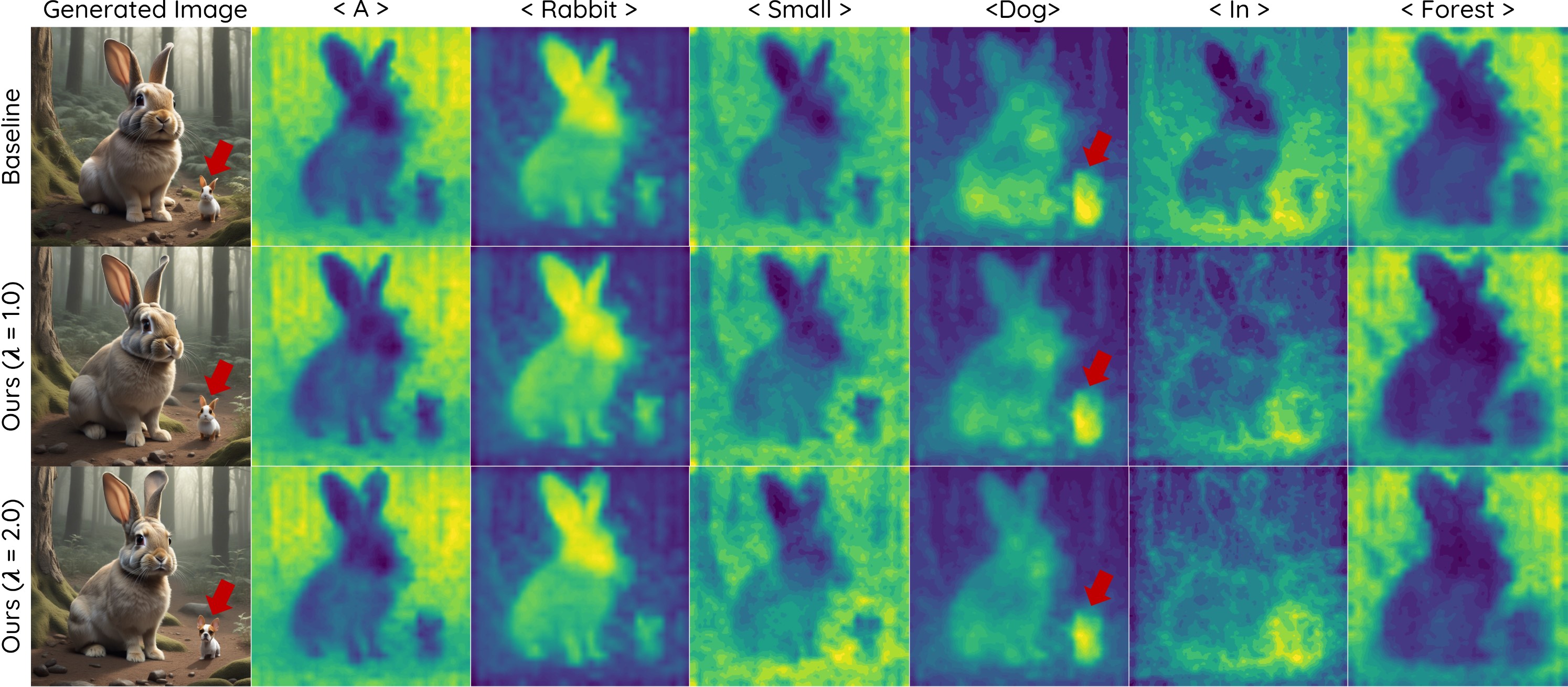}
\caption{Qualitative comparison of cross-attention average maps across all time steps. Top: Baseline. Middle: PLADIS (with $\lambda$ = 1) represent only use $\ent$ transformation. Bottom: PLADIS (with $\lambda$ = 2.0). Our PLADIS with $\lambda$ = 2.0 provides a more sparse and sharp correlation with each text prompt, especially "rabbit" and "dog." Furthermore, other approaches yield incorrect attention maps that highlight the space between the dog prompt and rabbit space. However, our method provides an exact attention map. }\label{fig:supple_cross}
\end{figure*}

\subsection{Analysis on Cross-Attention Map}
To analyze the effect of our proposed method in the cross-attention module, we directly visualize the cross-attention maps, as shown in Fig.~\ref{fig:supple_cross}. Each word in the prompt corresponds to an attention map linked to the image, showing that the information related to the word appears in specific areas of the image. We observe that the baseline (dense alignment with softmax) produces blurrier attention maps for the related words. Moreover, the generated image does not accurately reflect the text prompt of a "small dog," instead generating a "small rabbit." The cross-attention map highlights the small rabbit and a large rabbit nearby, associated with the dog prompt, resulting in poor text alignment.

When replacing the cross-attention with a sparse version, the maps become more sparse but still generate a "small rabbit" and incorrect attention maps. In contrast, our PLADIS produces both sparse and sharp attention maps compared to the baseline, and correctly aligns the attention maps with the given text prompts. As a result, PLADIS consistently improves text alignment and enhances the quality of generated samples across various interaction guidance sampling techniques and other distilled models.

\subsection{The Effect of Layer Group Selection}
To apply PLADIS in the cross-attention module, we incorporate it into all layers, including the down, mid, and up groups in the UNet. In SDXL, each group contains multiple layers; for example, the mid group has 24 layers, while the up group has 36 layers. To examine the effect of layer group selection, we focus on groups like the mid and up, instead of studying each layer \textit{ex.} the first layer in the up group. We conduct experiments by varying the groups for the application of PLADIS in the cross-attention module, as shown in Tab~\ref{tab_layer}.

Similar to previous ablation studies, we generate 5K samples from randomly selected data in the MS COCO validation set under CFG and PAG guidance. We observe that when applied to a single group, the up group has the most significant impact compared to others. However, in all cases, the use of PLADIS improves both generation quality and text alignment, as measured by FID and CLIPScore. Finally, combining all groups yields the best performance, confirming that no heuristic search for the target layer is necessary and validating our default configuration choice.

\begin{figure}
\centering
\includegraphics[width=0.9\linewidth]{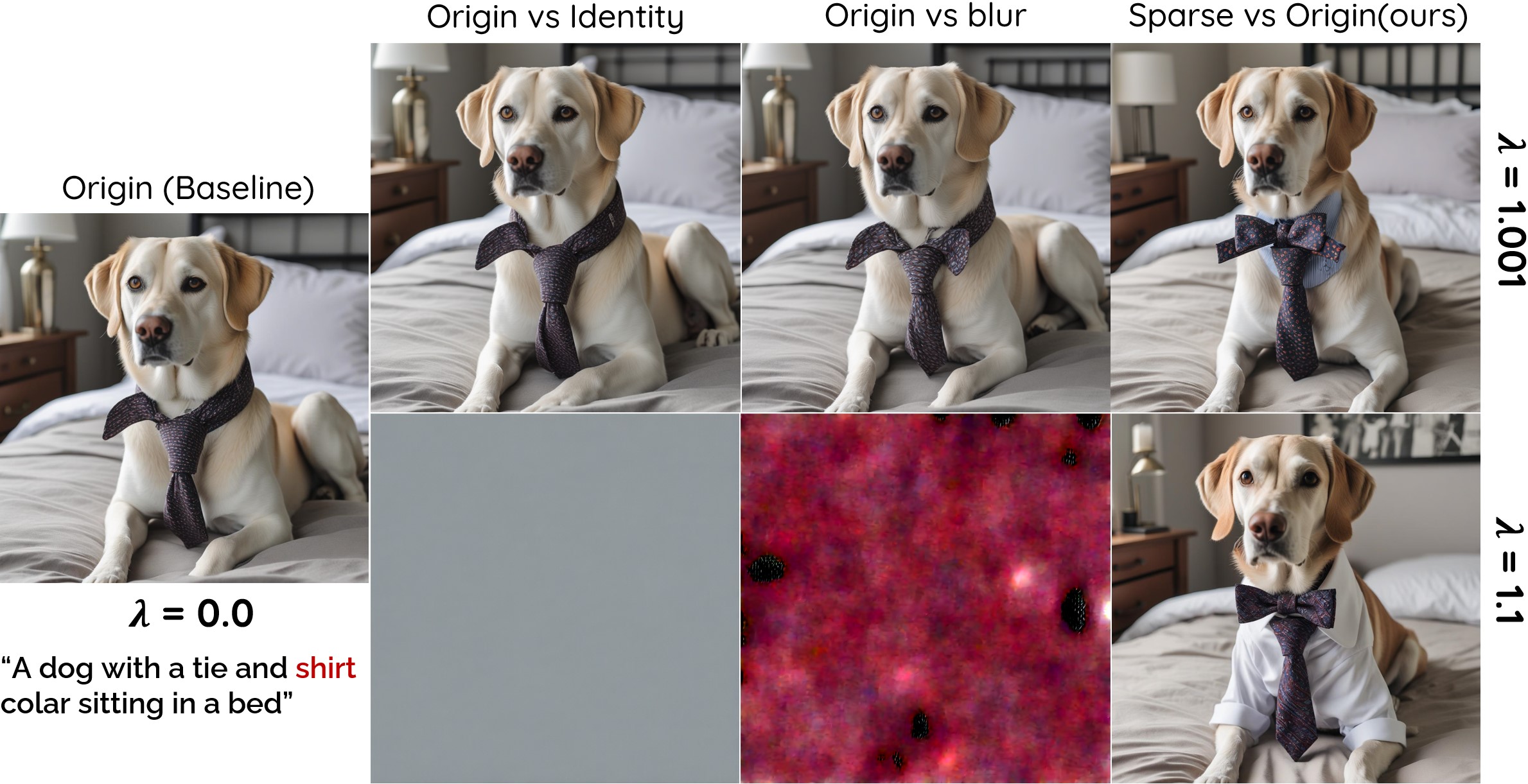}
\caption{In-model extrapolation results. Other perturbation approaches result in semantically degraded outputs even under minor extrapolation, whereas our method consistently improves generation quality. }\label{fig:in-model}
\end{figure}

\subsection{Two Extrapolation Strategies}

To validate our design choice, we investigate two types of extrapolation strategies using different attention mechanisms: in-model extrapolation and output-based extrapolation. For in-model extrapolation, we test perturbations using sparse attention, the identity matrix (PAG), and blurred attention maps (SEG). We observe that only sparse attention consistently improves performance under extrapolation, while other variants yield semantically meaningless outputs even under minor extrapolation (Fig.~\ref{fig:in-model}). This suggests that sparse attention operates as a valid energy landscape under Modern Hopfield dynamics, whereas identity or blurred attention matrices may affect diffusion outputs but fail to define coherent attention dynamics, ultimately leading to degraded generation quality.

We also explore output-based extrapolation using both sparse and dense attention variants. Although the output-based version of our method yields better performance than the baseline, it still underperforms compared to our in-model extrapolation while incurring higher inference costs (Tab.~\ref{tab_another}). These findings further support the efficiency and efficacy of our in-model extrapolation approach.

Our design is grounded in a principled integration of Hopfield retrieval dynamics and diffusion guidance. Specifically, we reinterpret extrapolation as a guidance process between a strong and a weak attention module inside the model. While diffusion-level extrapolation using blurred or identity attention may produce plausible outputs, such attention forms cannot act as valid components of attention dynamics. In contrast, sparse attention preserves the energy-based retrieval structure required for stable and interpretable in-model extrapolation.

\begin{table}[t]
\centering
\caption{Quantitative comparison on Geneval. Rows denote different methods, and columns denote guidance/backbone combinations.}
\vspace{-1em}
\label{tab:geneval2}
\resizebox{\linewidth}{!}{
\begin{tabular}{l c c c c c}
\toprule
Method & SDXL (CFG) & SDXL (CFG + PAG) & SDXL (CFG + SEG) & FLUX (schnell) & FLUX (dev) \\
\midrule
Baseline & 0.547 & 0.553 & 0.551 & 0.671 & 0.666 \\
Only Sparse & 0.581 & 0.571 & 0.582 & 0.694 & 0.676 \\
\rowcolor{green!10}Ours (Extrapolation) & 0.594 & \textbf{0.598} & \textbf{0.601} & \textbf{0.713} & \textbf{0.691} \\
\bottomrule
\end{tabular}
}
\end{table}

\subsection{Comparison with Sparse Attention Only}

To isolate the benefit of our extrapolation design beyond simply applying sparse attention, we conduct experiments on the Geneval benchmark, a reliable dataset for evaluating both text-image coherence and visual quality. As shown in Table~\ref{tab:geneval2}, across various guidance settings and even with a more challenging backbone (MMDiT), our method—extrapolation between sparse and dense attention—consistently outperforms both the baseline and the version using only sparse attention. These results further validate the effectiveness of our design choice.

\section{Additional Qualitative Results} \label{sec:add_example}

In this section, we present additional qualitative results to highlight the effectiveness and versatility of our proposed method, PLADIS, across various generation tasks and in combination with other approaches.

\paragraph{Comparison of Guidance Sampling with Our Method} Fig.~\ref{fig:supple_cfg}, \ref{fig:supple_pag}, and \ref{fig:supple_seg} provide qualitative results demonstrating interactions with existing guidance methods such as CFG, PAG, and SEG, respectively. By combining PLADIS with these guidance approaches, we observe a significant enhancement in image plausibility, particularly in text alignment and coherence with the given prompts, including improvements in visual effects and object counting. Through various examples of this joint usage, we demonstrate that PLADIS improves generation quality without requiring additional inference steps.

\paragraph{Comparison of Guidance-Distilled Models with Ours} Fig.~\ref{fig:supple_1step} and \ref{fig:supple_4step} present qualitative results from applying our method, PLADIS, to guidance-distilled models such as SDXL-Turbo~\cite{turbo}, SDXL-Lightening~\cite{light}, DMD2~\cite{dmd2}, and Hyper-SDXL~\cite{hyper}, for both 1-step and 4-step cases. Notably, PLADIS significantly enhances generation quality, removes unnatural artifacts, and improves coherence with the given text prompts, all while being nearly cost-free in terms of additional computational overhead.

\paragraph{Ablation Study on Scale $\lambda$} Fig.~\ref{fig:supple_ablation} shows a visual example of conditional generation with controlled scale $\lambda$. We generate samples using a combination of CFG and PAG, or CFG and SEG. For the ablation study, all other guidance scales are fixed, and only our scale $\lambda$ is adjusted. Consistent with the results shown in Sec~\ref{ablation}, a  scale $\lambda$ of 2.0 produces the best results in terms of visual quality and text alignment, which leads to our default configuration.

\paragraph{Ablation Study on $\alpha$ in $\ent$}
As discussed in Sec.~\ref{ablation}, PLADIS offers two options for choosing $\alpha$: 1.5 or 2. Fig.~\ref{fig:supple_alpha} provides a qualitative comparison between the baseline, $\alpha = 1.5$, and $\alpha = 2$. Empirically, we adopt $\alpha = 1.5$ as our default configuration. While PLADIS with $\alpha = 2$ improves generation quality and text alignment compared to the baseline (dense cross-attention), PLADIS with $\alpha = 1.5$ offers a more stable and natural enhancement in sample quality.

\clearpage

\begin{figure*}[ht!]
\centering
\includegraphics[width=0.95\linewidth]{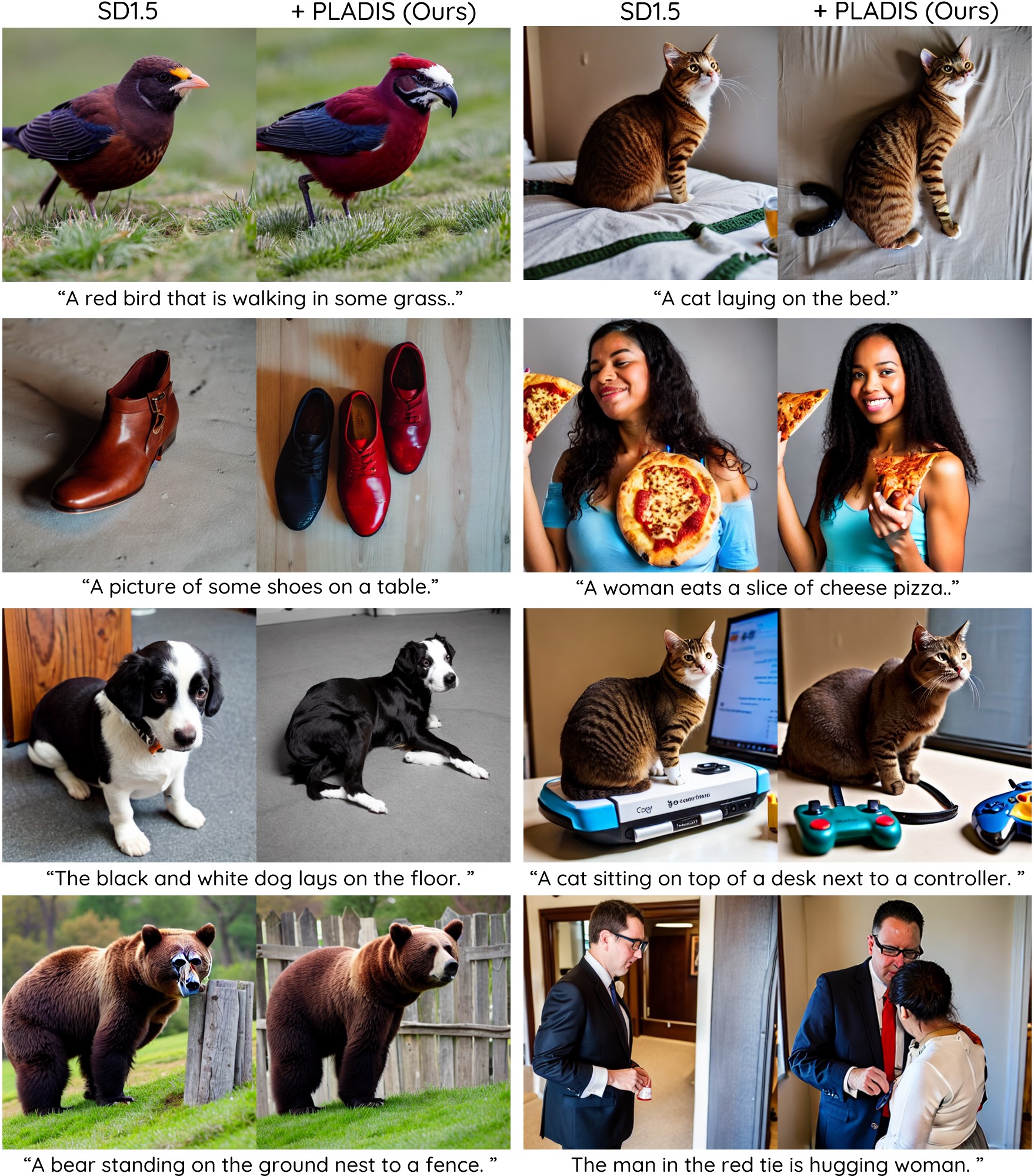}
\caption{Qualitative evaluation of Stable Diffusion 1.5 using our PLADIS method: PLADIS significantly boosts generation quality, strengthens alignment with the given text prompt, and generates visually compelling images.}\label{fig:supple_sd15}
\end{figure*}

\begin{figure*}[ht!]
\centering
\includegraphics[width=0.95\linewidth]{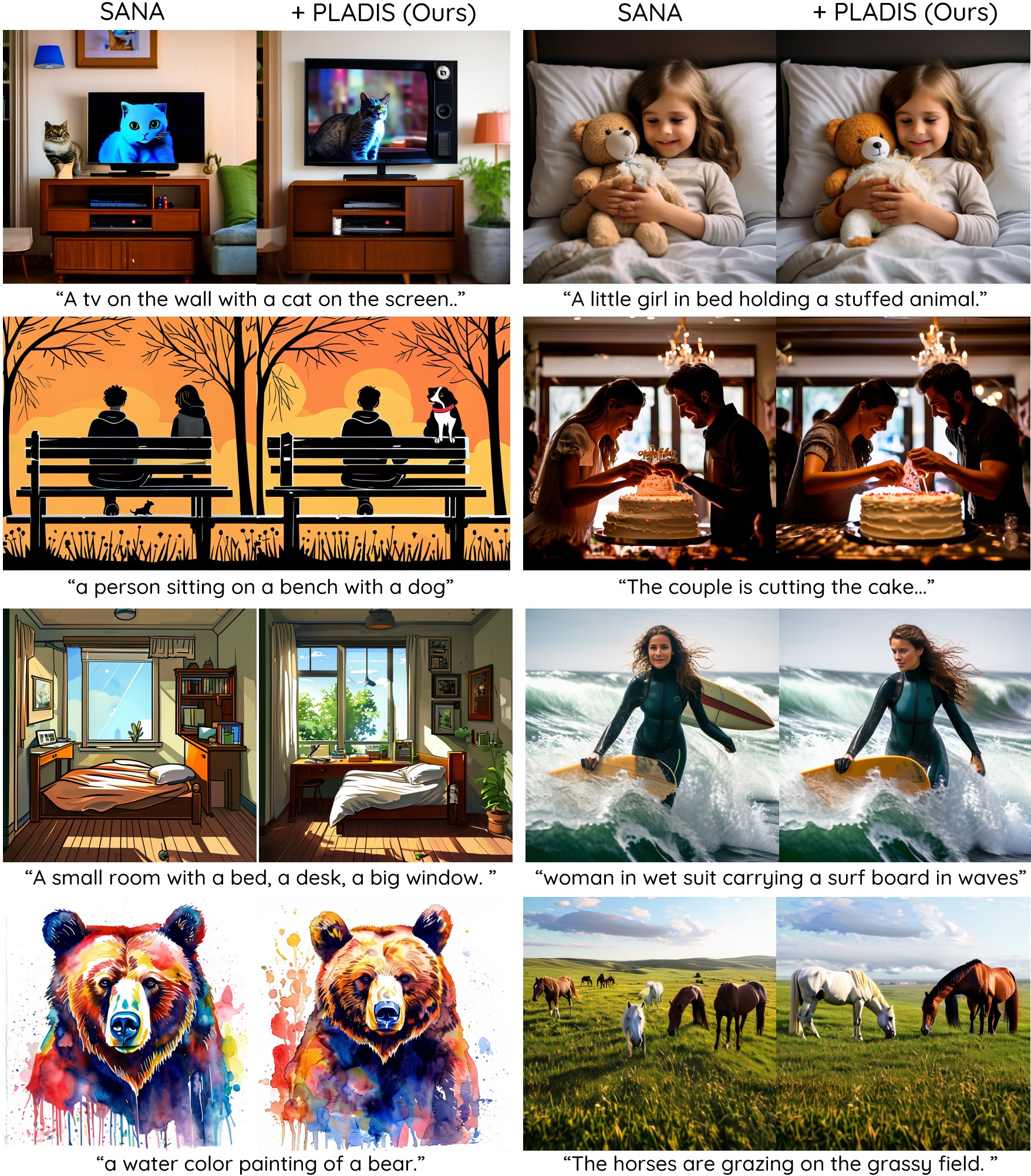}
\caption{Qualitative assessment of SANA~\cite{sana} with and without our PLADIS method: PLADIS notably improves generation quality, strengthens alignment with the provided text prompt, and produces visually striking images.}\label{fig:supple_sana}
\end{figure*}

\begin{figure*}[ht!]
\centering
\includegraphics[width=1\linewidth]{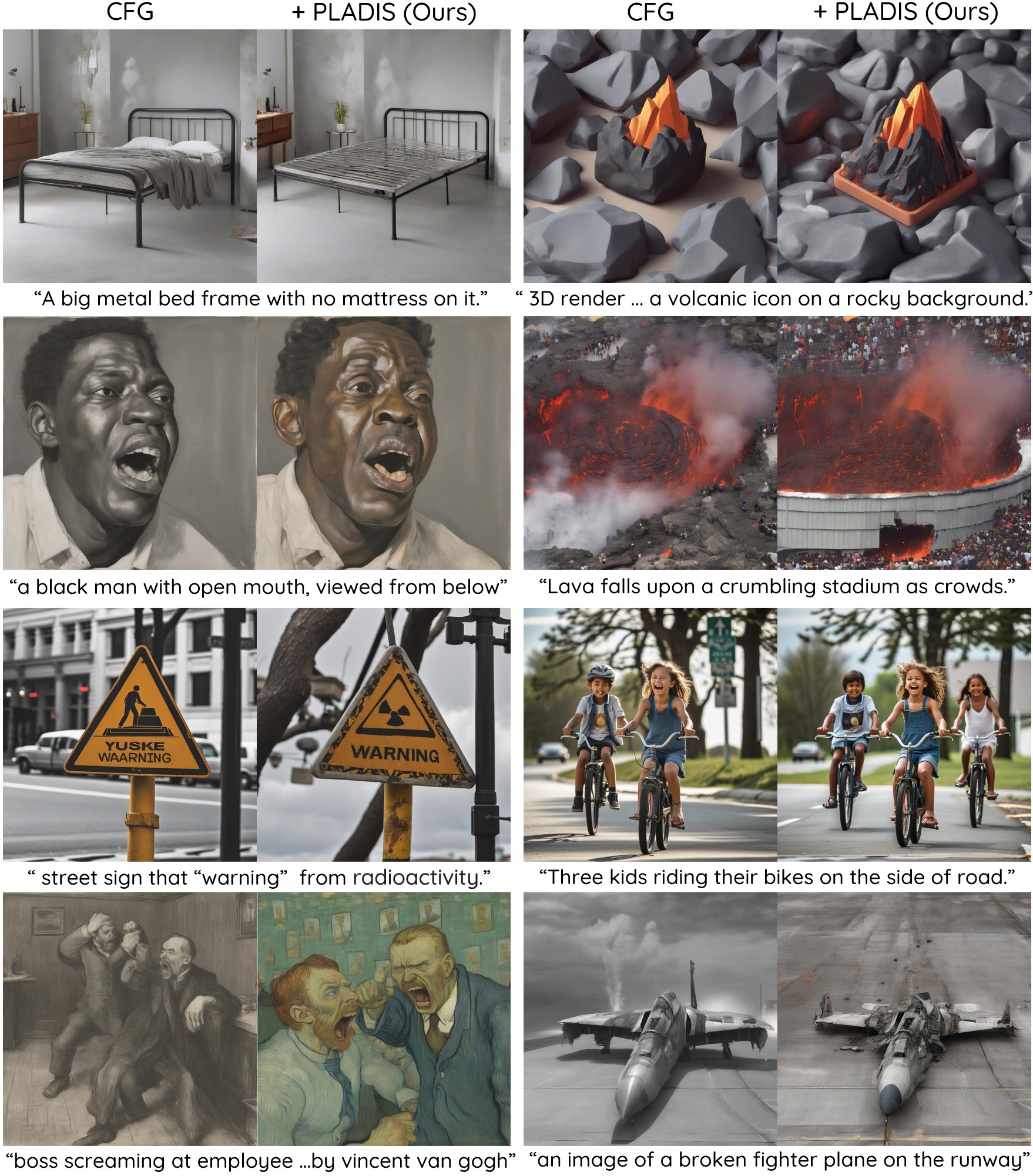}
\caption{Qualitative evaluation of the joint usage  CFG~\cite{CFG} with our method: CFG with PLADIS generates more plausible images with significantly improved text alignment based on the text prompt, without requiring additional inference.}\label{fig:supple_cfg}
\end{figure*}

\begin{figure*}[ht!]
\centering
\includegraphics[width=1\linewidth]{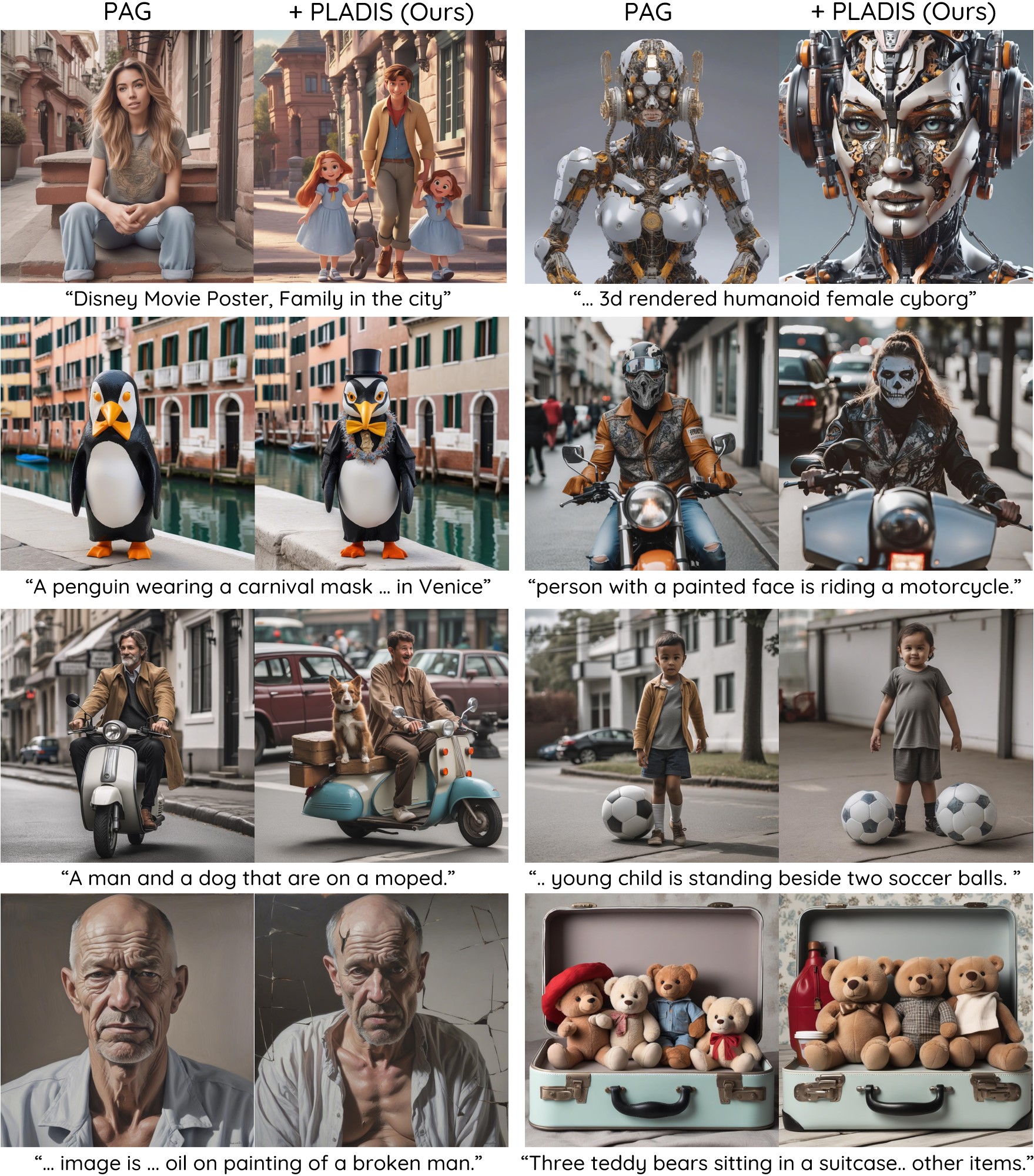}
\caption{Qualitative evaluation of the joint usage  PAG~\cite{PAG} with our method: Integrating PAG with PLADIS produces highly credible images with markedly enhanced correspondence to the text prompt, all achieved without any further inference steps.}\label{fig:supple_pag}
\end{figure*}

\begin{figure*}[ht!]
\centering
\includegraphics[width=1\linewidth]{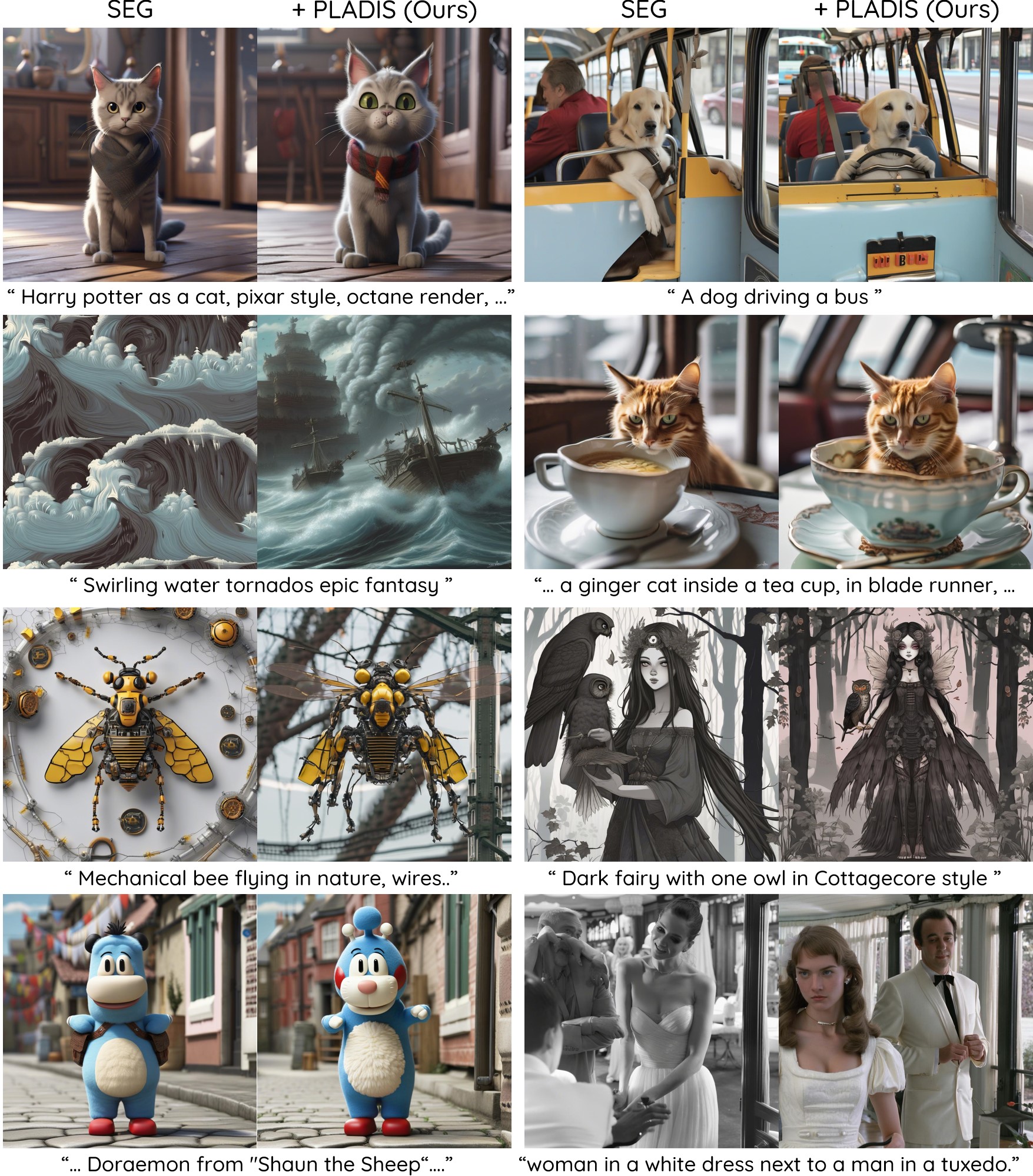}
\caption{Qualitative evaluation of the joint usage  SEG~\cite{SEG} with our method: The combination of SEG and PLADIS yields highly convincing image generations with substantially improved alignment to the given text prompt, accomplished without the need for additional inference.}\label{fig:supple_seg}
\end{figure*}

\begin{figure*}[ht!]
\centering
\includegraphics[width=0.95\linewidth]{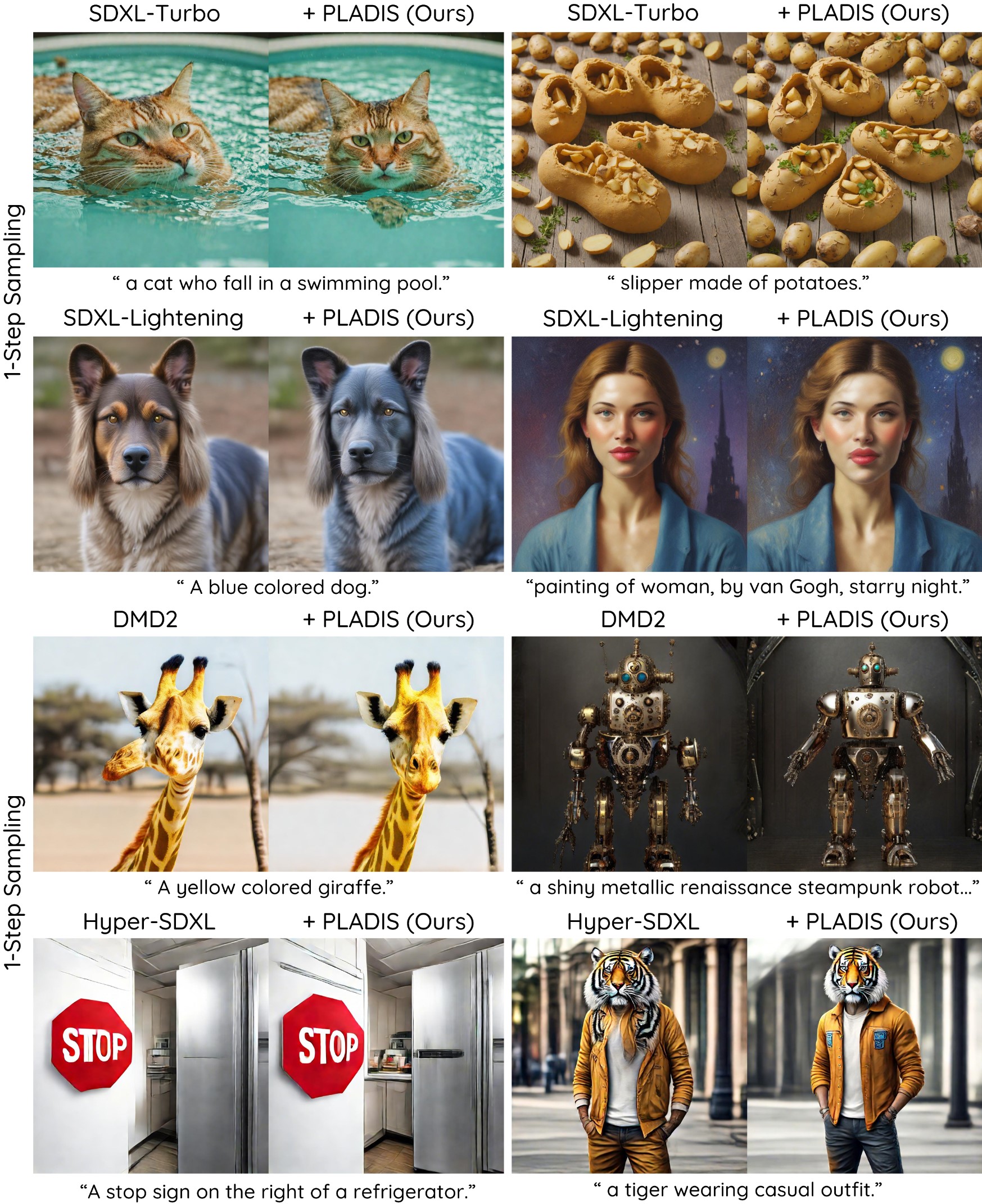}
\caption{Qualitative comparison of the guidance-distilled model with our PLADIS method for one-step sampling: Even with one-step sampling, our PLADIS enhances generation quality, improves coherence with the given text prompt, and produces visually plausible images. }\label{fig:supple_1step}
\end{figure*}

\begin{figure*}[ht!]
\centering
\includegraphics[width=0.95\linewidth]{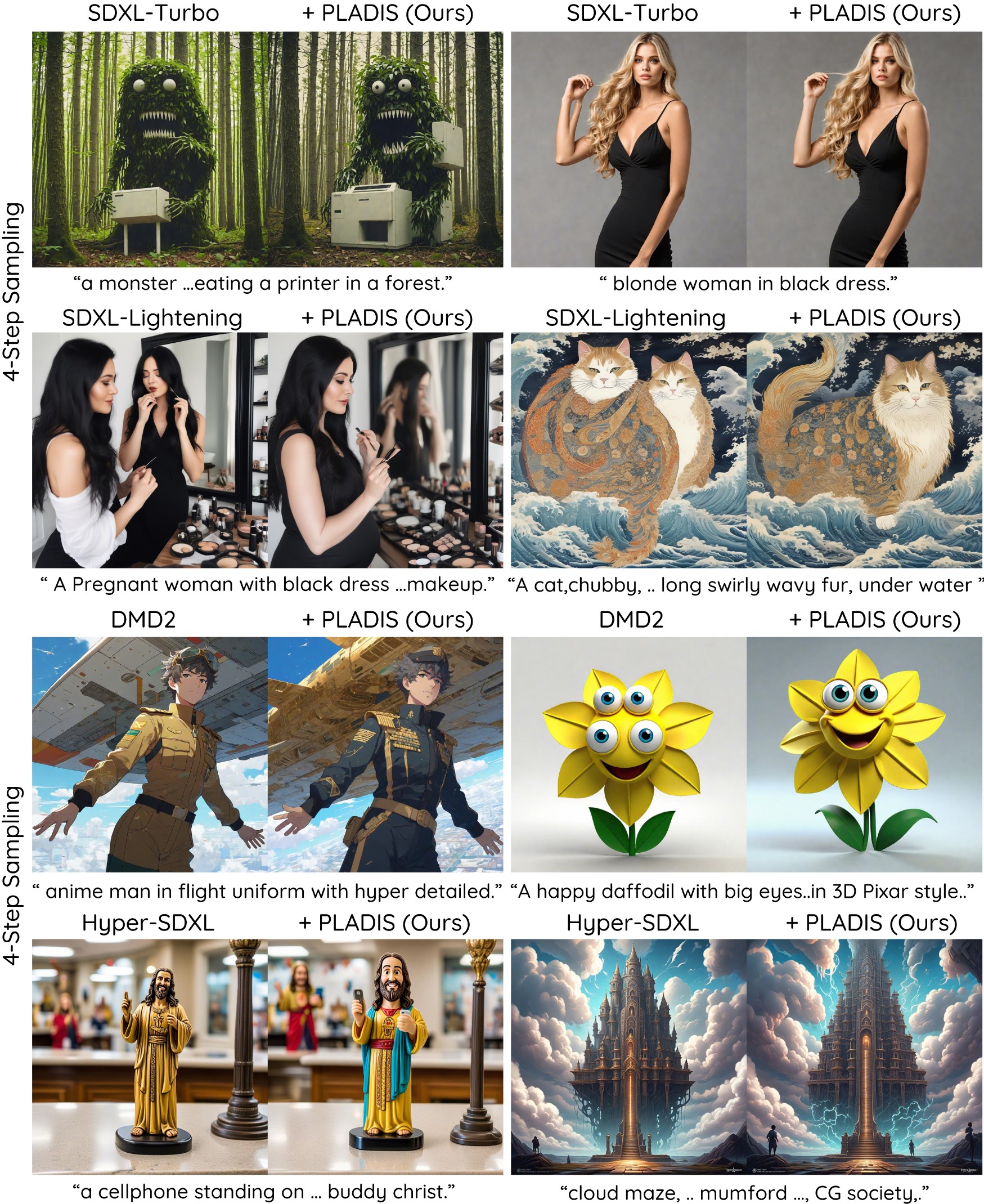}
\caption{Qualitative comparison of the guidance-distilled model using our PLADIS method for four-step sampling: In the case of the four-step sampling approach, PLADIS substantially improves generation quality, enhances alignment with the provided text prompt, and produces visually convincing images.}\label{fig:supple_4step}
\end{figure*}

\begin{figure*}[ht!]
\centering
\includegraphics[width=1\linewidth]{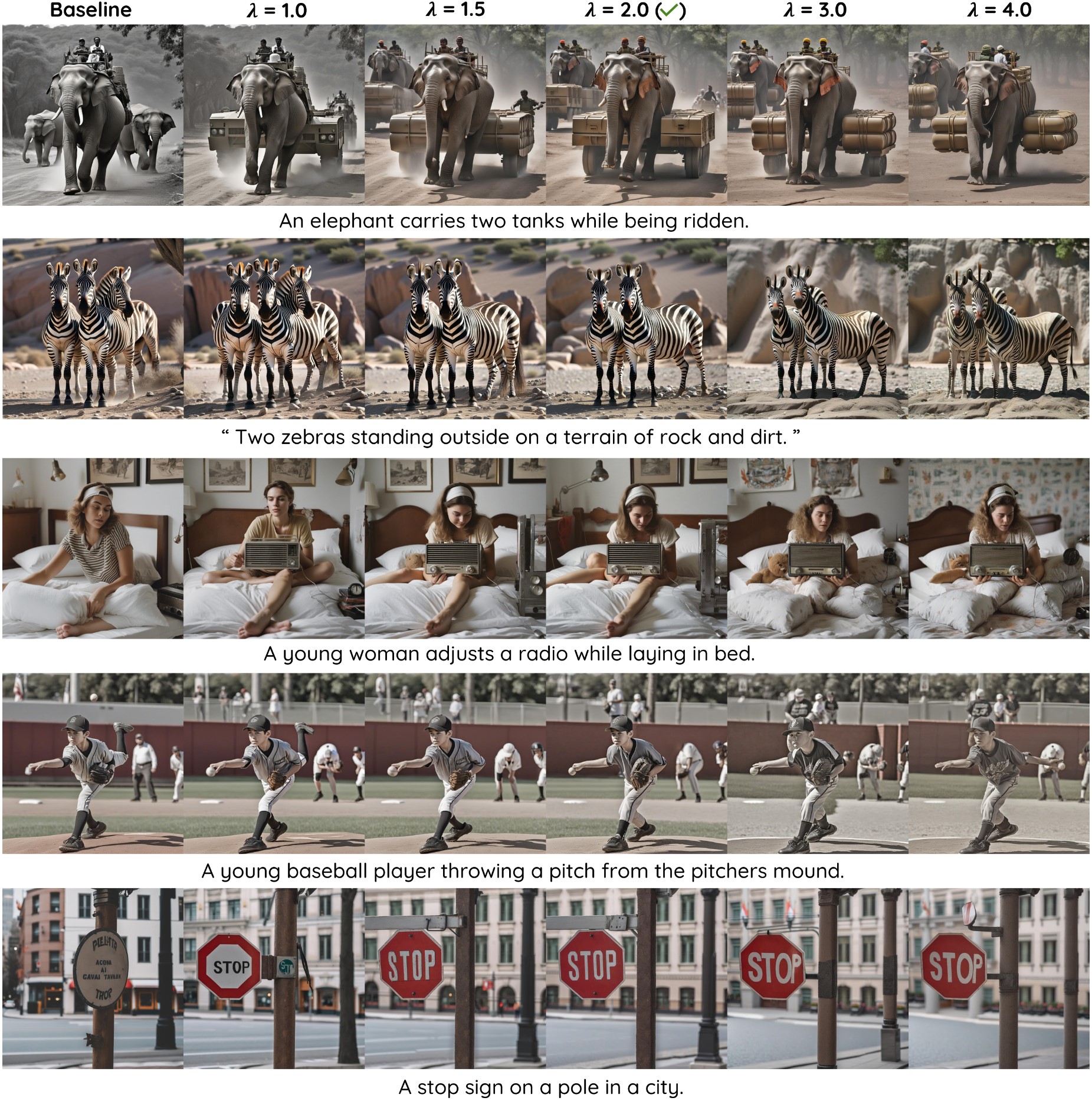}
\caption{Qualitative comparison by varying the scale $\lambda$: As $\lambda$ increases, the images display greater plausibility and improved text alignment. However, excessively high values lead to smoother textures and potential artifacts, similar to those found in CFG. The first two rows of images are generated using CFG and PAG, while the remaining rows are produced with CFG and SEG. When $\lambda$ is greater than 1, our PLADIS method is applied. In our configuration, $\lambda$ is set to 2.0.}\label{fig:supple_ablation}
\end{figure*}

\begin{figure*}[ht!]
\centering
\includegraphics[width=1\linewidth]{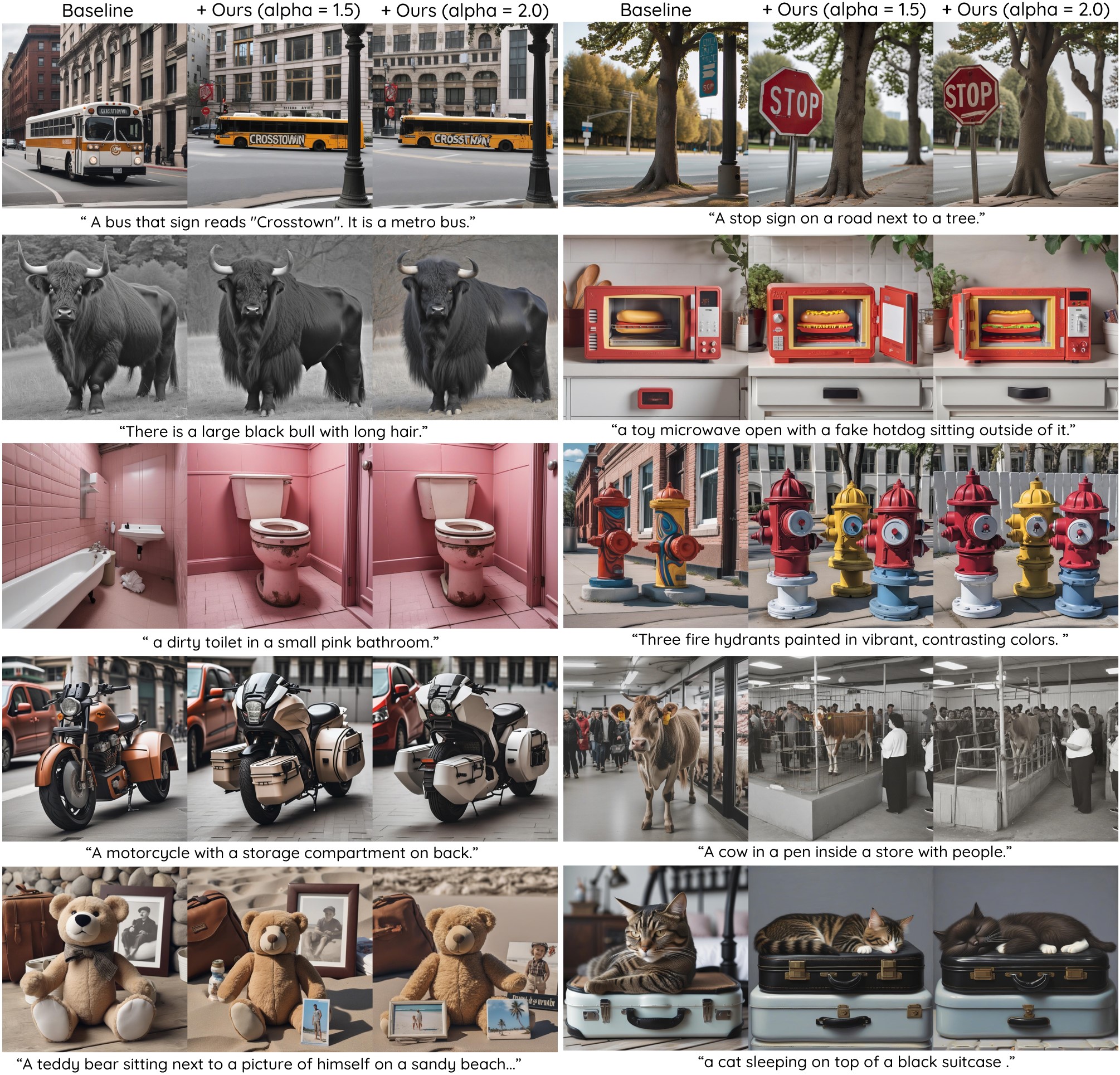}
\caption{Qualitative comparison by $\alpha$ in PLADIS: Although PLADIS with $\alpha = 2$ also sifgnificantly improves generation quality and text alignment compared to the baseline (dense cross-attention), PLADIS with $\alpha = 1.5$ offers a more robust and coherence given text prompts, leads to our base configuration as $\alpha = 1.5$.}\label{fig:supple_alpha}
\end{figure*}

\end{document}